\documentclass[10pt]{iopart}

\expandafter\let\csname equation*\endcsname\relax
\expandafter\let\csname endequation*\endcsname\relax
\usepackage{amsmath}

\usepackage{mathrsfs}
\usepackage{lipsum}
\usepackage{amsfonts}
\usepackage{graphicx}
\usepackage{epstopdf}
\usepackage{algorithmic}
\usepackage{array}
\usepackage{textcomp}
\usepackage{stfloats}
\usepackage{url}
\usepackage{verbatim}
\usepackage{graphicx}
\usepackage{cite}
\usepackage{subfigure}
\usepackage{amssymb}
\usepackage{booktabs}
\usepackage{algorithm}
\usepackage{float}
\usepackage{xcolor}
\usepackage{multirow}
\usepackage{bbm}
\usepackage{dsfont}
\usepackage{bm}
\usepackage{tikz}
\usepackage{lmodern}
\usetikzlibrary{positioning}
\usepackage{pstring}[pgf]

\usepackage{hyperref}
\hypersetup{
    colorlinks=true,
    linkcolor=blue,
    filecolor=blue,
    urlcolor=blue,
    citecolor=blue
}

\usepackage{ntheorem}
\newtheorem{theorem}{Theorem}
\newtheorem*{proof}{Proof}

\newtheorem{proposition}{Proposition}
\renewcommand{\qedsymbol}{\hfill \ensuremath{\square}}

\eqnobysec

\begin{document}

\title[Learning to Sample and Reconstruct for Accelerated MRI]{L2SR: Learning to Sample and Reconstruct for Accelerated MRI via Reinforcement Learning} 


\author{Pu Yang$^1$ and Bin Dong$^2$}

\address{$^1$ the School of Mathematical Sciences, Peking University, Beijing, P.R.China}
\address{$^2$ Beijing International Center for Mathematical Research, Peking University, Beijing, P.R.China}
\ead{dongbin@math.pku.edu.cn}
\vspace{10pt}
\begin{indented}
\item[]Submitted 29 October 2023
\end{indented}

\begin{abstract}
    Magnetic Resonance Imaging (MRI) is a widely used medical imaging technique, but its long acquisition time can be a limiting factor in clinical settings. To address this issue, researchers have been exploring ways to reduce the acquisition time while maintaining the reconstruction quality. Previous works have focused on finding either sparse samplers with a fixed reconstructor or finding reconstructors with a fixed sampler. However, these approaches do not fully utilize the potential of joint learning of samplers and reconstructors. In this paper, we propose an alternating training framework for jointly learning a good pair of samplers and reconstructors via deep reinforcement learning (RL). In particular, we consider the process of MRI sampling as a sampling trajectory controlled by a sampler, and introduce a novel sparse-reward Partially Observed Markov Decision Process (POMDP) to formulate the MRI sampling trajectory. Compared to the dense-reward POMDP used in existing works, the proposed sparse-reward POMDP is more computationally efficient and has a provable advantage. Moreover, the proposed framework, called L2SR (Learning to Sample and Reconstruct), overcomes the training mismatch problem that arises in previous methods that use dense-reward POMDP. By alternately updating samplers and reconstructors, L2SR learns a pair of samplers and reconstructors that achieve state-of-the-art reconstruction performances on the fastMRI dataset. Codes are available at \url{https://github.com/yangpuPKU/L2SR-Learning-to-Sample-and-Reconstruct}. 
\end{abstract}
\vspace{2pc}
\submitto{\IP}

\section{Introduction}
\label{sec:intro}
Magnetic resonance imaging (MRI) is a valuable medical imaging tool that does not involve ionizing radiation. However, it suffers from slow data acquisition, which limits its widespread clinical use.

To accelerate MRI, Compressed sensing (CS) \cite{candes2006robust} techniques have been introduced to speed up the MRI scanning process by reducing the number of measurements. Traditional CS-MRI methods \cite{lustig2007sparse, lustig2008compressed} first subsample the measurements in the Fourier domain with heuristic sampling policies and then reconstruct MR images by solving an ill-posed inverse problem. These optimization-based methods are interpretable and with theoretical guarantees. Although they have notably reduced acquisition time compared to commercial MRI, they still have room for improvement. 

Recent advances in Deep Learning (DL) \cite{8329428, hyun2018deep, 8233175, 8327637, NIPS2016_1679091c, hammernik2018learning, Zhang_2018_CVPR} have led to significant improvements in MRI reconstruction by using DL models as reconstructors. These models use heuristic sampling policies similar to CS-MRI and are trained on specific datasets. However, the quality of the sampling policies and reconstructors can be quantified by examining their acceleration-quality trade-off. The heuristic sampling policies may have a limited acceleration-quality trade-off, as they cannot adaptively decide the sampling strategy according to the anatomical structures of patients.

To overcome this, some approaches learn an adaptive sampling policy which designs a specific sampling pattern for each individual image based on its unique characteristics, rather than by human design. One such approach views the sequential sampling in the Fourier domain as a feedback control process, formulates the sampling trajectory as a Partially Observed Markov Decision Process (POMDP) and solves it via deep Reinforcement Learning (RL). Existing dynamic sampling methods \cite{NEURIPS2020_daed2103, pineda2020active, Liu2022ActivePS} propose using a dense-reward POMDP to learn parametric samplers with a fixed pre-trained reconstructor via deep RL. However, these methods suffer from the training mismatch, where the sampling pattern used to pre-train the reconstructor is different from the pattern given by the RL-learned sampler. This can lead to suboptimal performance in the learned pair of samplers and reconstructors.

In this paper, we propose a novel alternating training framework, L2SR (Learning to Sample and Reconstruct), which incorporates an innovative sparse-reward POMDP formulation to facilitate the joint optimization of MRI sampling policies and reconstruction models. In contrast to previous efforts \cite{NEURIPS2020_daed2103, pineda2020active, Liu2022ActivePS} our approach addresses the training mismatch issue, which emerges when samplers and reconstructors are optimized as isolated subproblems rather than focusing on the comprehensive joint problem. Unlike dense-reward POMDPs, our sparse-reward formulation uniquely disentangles the sampling and reconstruction stages by deliberately eliminating intermediate reconstructions during the sampling process. The final reconstruction occurs only at the end, leading to a unique joint sampler-reconstructor optimization problem. This problem can be solved by alternately optimizing the sampler (with the reconstructor kept constant) through RL, and the reconstructor (with the sampler kept constant) through back-propagation. Our experiments on fastMRI dataset \cite{zbontar2018fastmri} demonstrates L2SR's capability to enhance MRI reconstruction quality.

The main contributions of this paper are:
\begin{itemize}
    \item We formulate the MRI sampling and reconstruction process as a joint optimization problem of learning optimal MRI samplers and reconstructors, a novel formulation not explored in previous study. This formulation could potentially enhance the trade-off between acceleration and quality. 
    \item We propose a novel sparse-reward POMDP formulation for MRI sampling that addresses three key issues with dense-reward POMDPs: (1) It is more computationally efficient for inference by avoiding intermediate reconstructions; (2) It avoids the distributional mismatch issue, which is proven to have an advantage in reconstruction performance; (3) It provides a viable approach to tackle the joint sampler-reconstructor optimization problem. 
    \item We develop an alternating training framework called L2SR that leverages the sparse-reward POMDP to jointly optimize samplers and reconstructors. A key novelty is that L2SR overcomes the training mismatch issue that arises due to the decomposition of the joint optimization problem into separate sampler and reconstructor subproblems in prior works. 
    \item Our experiments demonstrate the state-of-the-art reconstruction performance on the fastMRI benchmark, validating the benefits of our proposed solutions for the joint optimization problem. 
\end{itemize}

The remainder of this paper is organized as follows. In \sref{sec:relat}, we discuss related works. In \sref{sec:prob}, we provide notations of accelerated MRI and formulate optimization problems. In \sref{sec:method}, we delve into the issues inherent to the design of the dense-reward POMDP, followed by the introduction of the sparse-reward POMDP, along with its associated dynamic sampling training framework and alternating training framework. Experimental results are presented in \sref{sec:exper}, and finally, we conclude our findings in \sref{sec:concl}.

\section{Related Work}
\label{sec:relat}
In this section, we briefly review prior works on sampling and reconstruction in accelerated MRI. 

\subsection{Learning Reconstructors with Fixed Heuristic Sampling Policy}
\label{subsec:relate-recon}
Existing heuristic sampling policies for MRI reconstruction \cite{lustig2007sparse, vasanawala2011practical, chauffert2014variable} include uniform density random sampling, variable density sampling, Poisson-disc sampling, equispaced sampling, and continuous-trajectory variable density sampling. These sampling strategies are very simple and easy to implement. Methods of learning reconstructors with a fixed heuristic sampling policy can be divided into two categories: optimization-based methods and DL-based methods.

\textbf{Optimization-based methods} use constrained optimization problems to obtain the reconstruction image. These problems minimize an $\ell_1$-norm regularization term for sparsity while constraining the $\ell_2$-norm error in the Fourier domain for accuracy (see Equation (3) in \cite{lustig2007sparse}). The regularization term is often referred to as Total Variation (TV) \cite{rudin1992nonlinear} or Wavelet \cite{chan2003wavelet}. These optimization problems can be solved using iterative optimization methods such as Iterative Soft Thresholding Algorithm (ISTA) \cite{daubechies2004iterative} and Iterative Reweighted Least Squares (IRLS) \cite{daubechies2010iteratively}. While these methods are interpretable and have theoretical guarantees, they depend on human-designed image priors and are not adaptive to specific datasets.

\textbf{DL-based methods} use DL models as reconstructors for MRI, significantly improving reconstruction quality. Unet-based network models \cite{8329428, hyun2018deep} and GAN-based network models \cite{8233175, 8327637} learn an end-to-end mapping from sampled measurements to reconstruction images. Unrolled dynamic models such as ADMM-net \cite{NIPS2016_1679091c}, Variational-net \cite{hammernik2018learning}, and ISTA-net \cite{Zhang_2018_CVPR} first unroll an iterative algorithm to form the backbone network architecture and then replace some of its operators with neural networks. These methods are data-driven, and hence therefore be adapted to specific datasets. However, compared to optimization-based methods, DL-based methods are less interpretable and may have generalization issues.

\subsection{Learning Dynamic Sampling with Fixed Reconstructors via Deep RL}
\label{subsec:relat-dynamic}
Heuristic sampling policies may not be adapted to specific datasets, which limits their acceleration-quality trade-off. Since the MRI sampling process is carried out over time and can be considered as a finite sequential decision-making process, a dynamic sampling method was proposed to learn a sampler with a fixed pre-trained reconstructor via deep RL. In Computed Tomography (CT) scanning, this method was first proposed by \cite{1930-8337_2022_1_179}, and a similar idea was later proposed in the MRI region. These methods \cite{NEURIPS2020_daed2103, pineda2020active, Liu2022ActivePS} formulate the MRI sampling trajectory as a Partially Observed Markov Decision Process (POMDP) (see \sref{subsec:optim}) and learn parametric samplers with a fixed pre-trained reconstructor by modern deep RL algorithms, such as DDQN \cite{van2016deep} and policy gradient \cite{baxter2001infinite}. RL enables the learning of a non-greedy sequential decision function without relying on hand-designed policies. These methods use the RL-learned policies to guide the scanning based on the measurements collected in the preceding stage, thus achieving personalized scanning.

\subsection{Jointly Learning Samplers and Reconstructors via Back-Propagation}
\label{subsec:relate-joint}
If we fix a sampler to find an optimal reconstructor or fix a reconstructor to find an optimal sampler, we may obtain suboptimal pairs of samplers and reconstructors. Therefore, recent works aim to learn a good pair of samplers and reconstructors for a better acceleration-quality trade-off. These works propose to jointly train samplers and reconstructors end-to-end using Back-Propagation (BP).

One line of works \cite{8353419, bahadir2019learning, 9054542, zhang2020extending, 9105133, wang2021b, 9053345, Gzc2019RethinkingSI} use a learnable column vector of probabilistic sub-sampling masks $\mathbf{p}$ as the sampler, which is independent of the specific image, and jointly train it together with a neural network-based reconstruction model $\mathcal{R}$. These works learn a better sampler $\mathbf{p}$ than human-designed heuristic sampling policies. However, these samplers are also fixed sampling patterns and not adaptive to each image. The ideas of these works can be illustrated by the following diagram:

\setlength\fboxsep{1pt}
$$
    \fbox{$\mathbf{p}$} \xrightarrow[\mathbf{y}]{\text{sample}} \mathbf{y}_T \xrightarrow{\fbox{$\mathcal{R}$}} \mathbf{x}_T
$$
where modules in the boxes are learnable and some notations are defined in \sref{subsec:notation} and can be quickly found through the notation index in \ref{appen-sec:notation}. 

Another line of works \cite{jin2019self, zhang2019reducing, van2021active, yin2021end} use a neural network-based model $\pi$ as the sampler, which is adaptive to sampled measurements. The sampler $\pi$ is jointly trained with a DL-based reconstructor $\mathcal{R}$ end-to-end via BP. These methods learn sampling strategies that are adaptive to each image. However, training the sampler and the reconstructor by BP can be challenging. On the one hand, gradients are passed through long trajectories leading to gradient vanishing or exploding. On the other hand, GPU memory limits the size of neural networks due to long trajectories. Therefore, these works often solve weakened forms of the joint optimization problem \eref{equ:dense-optim-constrain}. The ideas of these works can be illustrated by the following diagram:

$$
\pstr[][-15pt]{
    \nd (y0){\mathbf{y}_0}
    \nd \xrightarrow{\fbox{$\mathcal{R}$}} \mathbf{x}_0 \xrightarrow[\mathbf{y}]{\fbox{$\pi$}} (y1){\mathbf{y}_1}
    \nd \xrightarrow{\fbox{$\mathcal{R}$}} \mathbf{x}_1 \xrightarrow[\mathbf{y}]{\fbox{$\pi$}} (y2){\mathbf{y}_2}
    \nd \cdots (yT){\mathbf{y}_T} 
\arrow{y0}{y1}{30}{}{black}{}
\arrow{y1}{y2}{30}{}{black}{}
}
    \xrightarrow{\fbox{$\mathcal{R}$}} \mathbf{x}_T
$$
where modules in the boxes are learnable.

In this paper, we also aim to learn a pair of samplers and reconstructors jointly. Instead of training end-to-end by BP, we will design an RL-based joint training framework.

\section{Problem Formulation}
\label{sec:prob}
In \sref{subsec:notation}, we specify the notation of CSMRI. In \sref{subsec:optim}, we formulate the optimization problem. 

\subsection{Notation}
\label{subsec:notation}
In the context of this study, our focus is confined to the analysis of two-dimensional (2D) accelerated MRI in conjunction with one-dimensional (1D) vertical Cartesian sampling. We define a ground truth image, denoted by the matrix $\mathbf{x} \in \mathbb{R}^{N\times N}$, and the corresponding fully sampled measurements in the Fourier domain (referred to as k-space in MRI data), represented by the matrix $\mathbf{y} \in \mathbb{C}^{N \times N}$, obtained via the two-dimensional Discrete Fourier Transform, $\mathcal{F}$, i.e., $\mathbf{y} = \mathcal{F}(\mathbf{x})$. For fully-observed MRI scanning, the ground truth image can be reconstructed from the fully sampled measurements by Inverse Fourier Transform, i.e., $\mathbf{x} = \mathcal{F}^{-1}(\mathbf{y})$. 

Further, in accelerated MRI, we only observe partially sampled measurements. Mathematically, the process of sampling is realized by masking $\mathbf{y}$ using a binary mask matrix $\mathbf{M} \in \{0, 1\}^{N \times N}$, rendering the observation of $\mathbf{M} \odot \mathbf{y}$, where $\odot$ designates the Hadamard product. In this study, all binary mask matrices satisfy the property that each of their columns is either all 0's or all 1's, unless specified otherwise.

The process of reconstruction involves a specific mapping $\mathcal{R}: \mathbb{C}^{N \times N} \rightarrow \mathbb{R}^{N\times N}$, called the reconstructor, which maps (partially-observed) measurements to the reconstructed image end-to-end. In this paper, it is customarily formulated as a parametric neural network denoted by $\mathcal{R}(\cdot ; \theta_{\mathcal{R}})$. 

Additionally, we define a fixed heuristic sampling policy $\pi_t^{\text{h}}$ as a probability distributional function accepting the $N$-dimensional binary column vector $\mathbf{a}$ as an input. Specifically, $\pi_t^{\text{h}}$ sample $\mathbf{a}$ under the constraint that $\| \mathbf{a} \|_1 = t$, i.e., $\pi_t^{\text{h}}$ is constructed in such a way that $\pi_t^{\text{h}}(\mathbf{a}) = 0$ if $\| \mathbf{a} \|_1 \neq t$. Denote the mask matrix $\mathbf{M}^{\mathbf{a}} = \mathbf{1} \cdot \mathbf{a}^T$ with the property that $\| \mathbf{M}^{\mathbf{a}} \|_\infty = \| \mathbf{a} \|_1$. 

We give a list of notation index and where it is defined in \ref{appen-sec:notation}. 

\subsection{Optimization Problems}
\label{subsec:optim}
To get an optimal pair of samplers and reconstructors, we want to solve the following joint optimization problem based on dense-reward POMDP
\begin{equation} \label{equ:dense-optim-constrain}
\begin{split}
    \max_{\theta_\pi, \theta_{\mathcal{R}}} \quad & \mathbb{E}_{\mathbf{x}\sim \mathcal{D}} \mathbb{E}_{\pi} \mathrm{S}(\mathbf{x}_T, \mathbf{x}) \\
    \text{s.t.} \quad & \mathbf{y}_t = \mathbf{M}_t \odot \mathcal{F}(\mathbf{x}), \quad t = 0, 1, \cdots, T, \\
    & \mathbf{x}_t = \mathcal{R}(\mathbf{y}_t ; \theta_{\mathcal{R}}), \quad t = 0, 1, \cdots, T, \\
    & a_t \sim \pi(\cdot \mid \mathbf{x}_t ; \theta_\pi), \quad t = 0, 1, \cdots, T-1, \\
    & \mathbf{M}_{t+1} = \mathds{1}(\mathbf{M}_t + \mathbf{M}^{a_t}), \quad t = 0, 1, \cdots, T-1,
\end{split}
\end{equation}
where $\mathcal{D}$ signifies a dataset comprising ground truth images, $T$ denotes the length of sampling trajectory, $\mathrm{S}$ is an image similarity metric, $a_t \in \mathcal{A} = \{n\}_{n=1}^N$, the sampler $\pi: \mathbb{R}^{N\times N} \rightarrow \Delta^N$ is the N-dimensional discrete conditional probability distribution function, $\Delta^N = \{ \mathbf{u} \in [0,1]^N \mid \sum_i u_i = 1 \}$ denotes the $N$-dimensional probability simplex, $\mathbf{M}_0$ is an initial binary mask matrix, and $\mathds{1}(\cdot)$ is a matrix indicator function. For simplicity, we denote the objective function $\mathbb{E}_{\pi} \mathrm{S}(\mathbf{x}_T, \mathbf{x})$ in \eref{equ:dense-optim-constrain} as $J_T^{\text{dense}}(\mathbf{x}; \pi, \mathcal{R})$. 

One of the challenges in solving this joint optimization problem directly is its complexity. One approach is to train both $\theta_\pi$ and $\theta_\mathcal{R}$ end-to-end by BP, as mentioned in \sref{subsec:relate-joint}. However, training may fail due to the long sampling trajectory, and existing BP methods often solve weakened forms of \eref{equ:dense-optim-constrain}. Another approach is to use reinforcement learning (RL). However, the existing dense-reward partially observable Markov decision process (POMDP) may not be suitable for end-to-end joint training. When training the sampler via RL, we must keep the environment unchanged, and thus the reconstructor as a part of the environment must also remain the same. This contradicts the end-to-end joint training. Additionally, it is difficult to design a multi-agent POMDP for the two learnable agents, the sampler and the reconstructor.

DL-based MRI methods can obtain a better reconstructor with a fixed heuristic sampling policy $\pi^{\text{h}}$ by solving the following suboptimization problem of \eref{equ:dense-optim-constrain}:
\begin{equation}
\max_{\theta_\mathcal{R}} \quad \mathbb{E}_{\mathbf{x} \sim \mathcal{D}}\mathbb{E}_{\mathbf{a}\sim\pi^{\text{h}}} \left[ \mathrm{S} (\mathcal{R}(\mathbf{M}^\mathbf{a} \odot \mathcal{F}(\mathbf{x}); \theta_\mathcal{R}), \mathbf{x}) \right].
\label{equ:heuristic-suboptim}
\end{equation}

Dynamic sampling methods can obtain better samplers by solving the following suboptimization problem of \eref{equ:dense-optim-constrain} with a fixed pre-trained reconstruction model $\mathcal{R}(\cdot)$:
\begin{equation}
\max_{\theta_{\pi}} \quad \mathbb{E}_{\mathbf{x} \sim \mathcal{D}} J_T^{\text{dense}}(\mathbf{x}; \pi(\cdot; \theta_\pi), \mathcal{R}).
\label{equ:dense-suboptim}
\end{equation}
To solve \eref{equ:dense-suboptim}, they formulate the sequential sampling process as the following \textit{dense-reward POMDP} (illustrated in \fref{subfig:dense-reward-MDP}): 
\begin{itemize}
\item Observation
\begin{equation} \label{equ:dense-y}
    \mathbf{y}_t = \mathbf{M}_t \odot \mathbf{y}, \quad t = 0, 1, \cdots, T,
\end{equation}
\begin{equation} \label{equ:dense-x}
    \mathbf{x}_t = \mathcal{R}(\mathbf{y}_t), \quad t = 0, 1, \cdots, T.
\end{equation}
\item Action set $\mathcal{A}$ and action $a_t \in \mathcal{A}$
\begin{equation} \label{equ:dense-a}
    a_t \sim \pi(\cdot \mid \mathbf{x}_t), \quad t = 0, 1, \cdots, T-1.
\end{equation}
\item Transition
\begin{equation}
    \mathbf{M}_{t+1} = \mathds{1}(\mathbf{M}_t + \mathbf{M}^{a_t}), \quad t = 0, 1, \cdots, T-1, 
\end{equation}
and $\mathbf{y}_{t+1}$ and $\mathbf{x}_{t+1}$ is calculated by \eref{equ:dense-y} and \eref{equ:dense-x}. 
\item Reward
\begin{equation}
    r_t = \mathrm{S}(\mathbf{x}_t, \mathbf{x}) - \mathrm{S}(\mathbf{x}_{t-1}, \mathbf{x}), \quad t = 1, \cdots, T.
\label{equ:dense-reward}
\end{equation}
\item Discount factor 
\begin{equation}
    \gamma \in [0,1]. 
\end{equation}
\end{itemize}

\begin{figure}
    \centering
    \subfigure[dense-reward POMDP]{
        \centering
        \includegraphics[width=0.4\columnwidth]{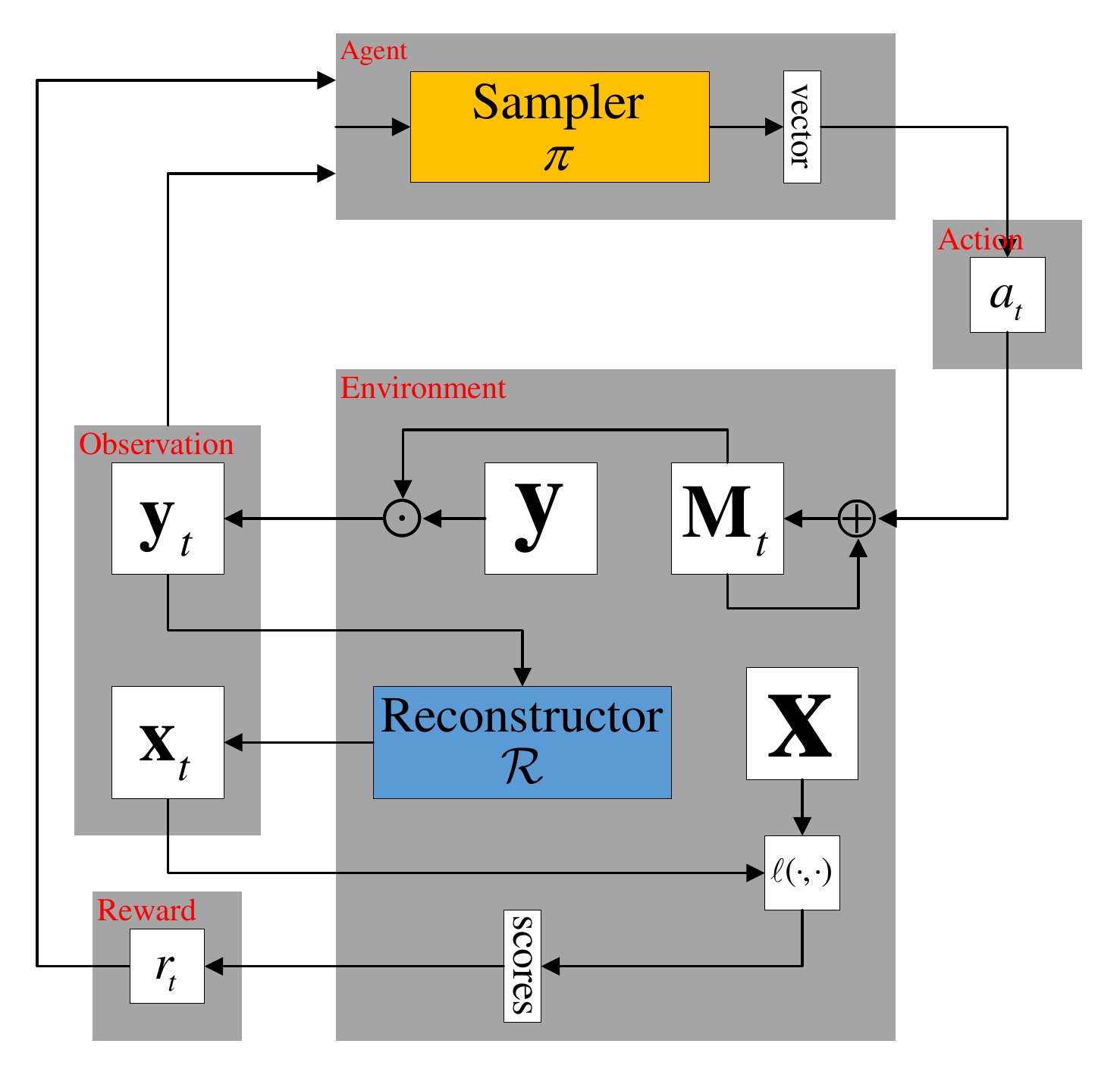}
        \label{subfig:dense-reward-MDP}
    }
    \subfigure[sparse-reward POMDP]{
        \centering
        \includegraphics[width=0.4\columnwidth]{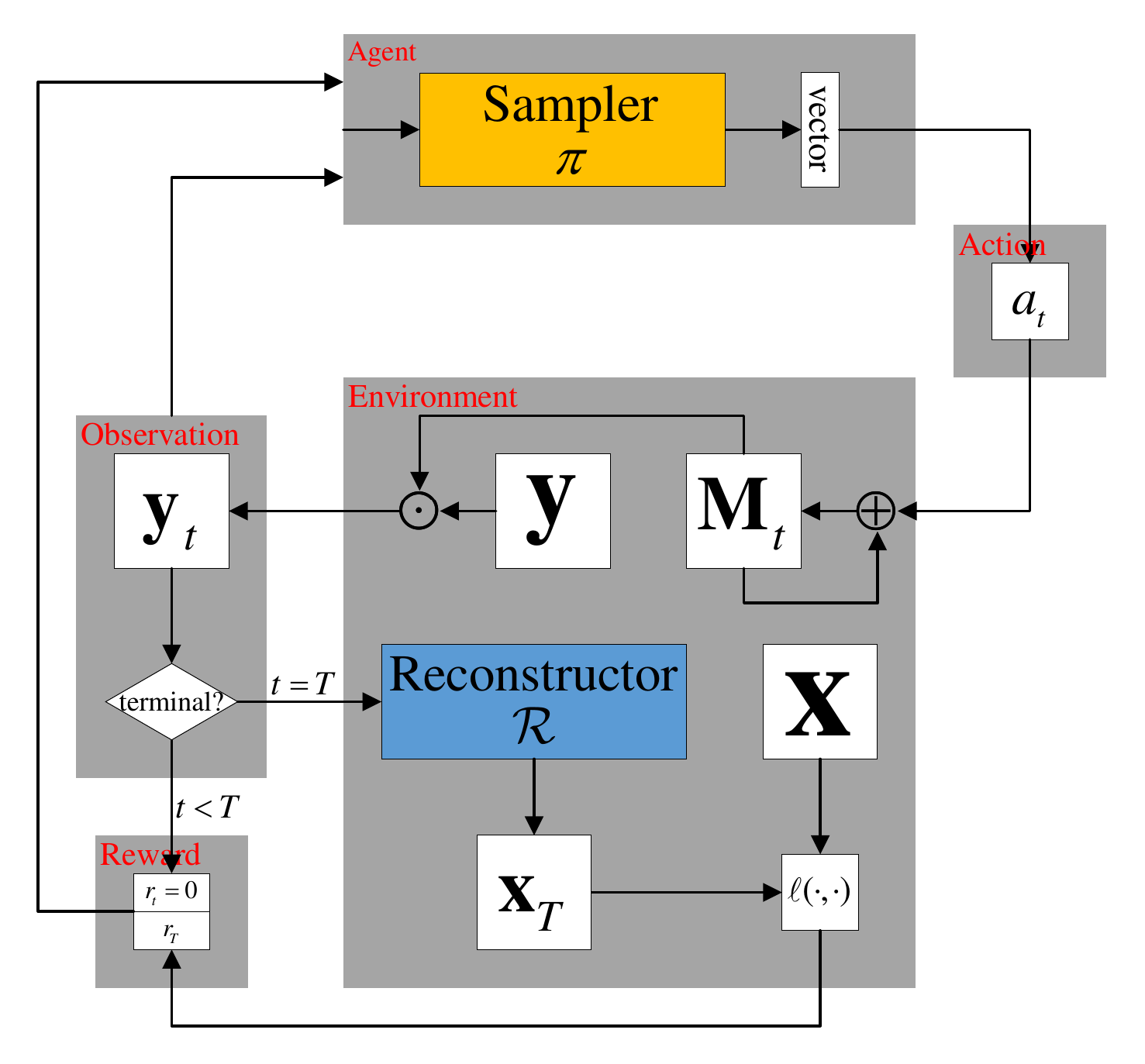}
        \label{subfig:sparse-reward-MDP}
    }
    \caption{Diagram of the dense-reward and sparse-reward POMDP.}
    \label{fig:MDP}
\end{figure}
These methods solve \eref{equ:dense-suboptim} via deep RL with the above dense-reward POMDP. 

However, a recent study \cite{yin2021end} has revealed a problem with training when solving the suboptimization problem \eref{equ:dense-suboptim}. Specifically, there is a \textit{training mismatch} when the heuristic sampling policy used to pre-train the reconstructor is different from the policy learned through reinforcement learning. This mismatch leads to a suboptimal pair of samplers and reconstructors for the joint optimization problem \eref{equ:dense-optim-constrain}.

\section{Proposed Method}
\label{sec:method}
In this section, we first conduct an analysis of the dense-reward POMDP and identify some inherent issues in \sref{subsec:analysis}. Based on these insights, we then propose a novel sparse-reward POMDP in \sref{subsec:sparse-reward} to address the identified issues. Furthermore, we present two novel training frameworks: dynamic sampling training framework in \sref{subsec:dynamic-sampling} and alternate training framework in \sref{subsec:alternate}, both of which stem from the sparse-reward POMDP.

\subsection{The Issues with Dense-reward POMDP}
\label{subsec:analysis}
Apart from the training mismatch issue arising from tackling the sub-optimization problem, our comprehensive analysis of the dense-reward POMDP has revealed issues associated with the MDP design. Specifically, the dense-reward POMDP embeds reconstructions into the environment, particularly within transitions and rewards. This design introduces two main problems: high computational cost and distributional mismatch. These findings pave the way for the introduction of our sparse-reward POMDP.

\subsubsection{High Computational Cost}
\label{subsec:high-comput-cost}
Since the dense-reward POMDP involves reconstructions in the transition, it first performs reconstruction \eref{equ:dense-x} and then sampling \eref{equ:dense-a} at each step in the inference process. However, this approach is very time-consuming since it requires performing many reconstructions. The dynamic sampling method for CT scanning \cite{1930-8337_2022_1_179} uses the original sampling signals rather than reconstructed images as the input of the sampler, which is more computationally efficient. We aim to find a better POMDP that requires only one reconstruction in the inference process.

\subsubsection{Distributional Mismatch}
\label{subsec:distributional-mismatch}
Since the dense-reward POMDP involves reconstructions in the reward calculations \eref{equ:dense-reward}, it requires evaluating the similarity of the ground truth and the reconstruction image $\mathrm{S}(\mathbf{x}_t, \mathbf{x})$ at every intermediate state. Therefore, the dense-reward POMDP requires a reliable reconstructor for all the intermediate states, which means the reconstructor should be pre-trained with a mixture of heuristic sampling policies $\pi_\text{mix}^{\text{h}} = \sum_{t = \| \mathbf{M}_0 \|_\infty}^{\| \mathbf{M}_0 \|_\infty + T} c_t \pi_t^{\text{h}}$ where $\sum c_t = 1$, instead of only the terminal sampling policy $\pi_{T+\| \mathbf{M}_0 \|_\infty}^{\text{h}}$. However, as we will show later, pre-training with the mixture of policies may not improve the final reconstruction performance, i.e., the optimal value of the joint optimization problem \eref{equ:dense-optim-constrain}. We will justify this claim in theorem \ref{theorem:improved-dynamic}.

\subsection{Sparse-Reward POMDP}
\label{subsec:sparse-reward}
To improve inference efficiency and eliminate distributional mismatch, we propose a \textit{sparse-reward POMDP} that reconstructs after completing the entire trajectory and receives a non-zero reward after the final reconstruction. Specifically, the proposed sparse-reward POMDP is as follows (illustrated in \fref{subfig:sparse-reward-MDP}):
\begin{itemize}
\item Observation
\begin{equation} \label{equ:sparse-y}
\mathbf{y}_t = \mathbf{M}_t \odot \mathbf{y}, \quad t = 0, 1, \cdots, T. 
\end{equation}
\item Action set $\mathcal{A}$ and action $a_t \in \mathcal{A}$
\begin{equation}
    a_t \sim \pi(\cdot \mid \mathbf{y}_t), \quad t = 0, 1, \cdots, T-1, 
\end{equation}
where the sampler $\pi: \mathbb{C}^{N\times N} \rightarrow \Delta^N$ is the N-dimensional discrete conditional probability distribution function. 
\item Transition
\begin{equation}
    \mathbf{M}_{t+1} = \mathds{1}(\mathbf{M}_t + \mathbf{M}^{a_t}), \quad t = 0, 1, \cdots, T-1, 
\end{equation}
and $\mathbf{y}_t$ is calculated by \eref{equ:sparse-y}
\item Reward
\begin{equation}
    \left\{ 
    \begin{array}{l}
        r_t = 0, \quad t = 1, \cdots, T-1 \\
        r_T = \mathrm{S}(\mathcal{R}(\mathbf{y}_T), \mathbf{x}) \\
    \end{array}. 
\right.
\end{equation}
\item Discount factor
\begin{equation}
    \gamma = 1. 
\end{equation}
\end{itemize}

It derives a sparse-reward joint optimization problem:
\begin{equation} \label{equ:sparse-optim-constrain}
\begin{split}
    \max_{\theta_\pi, \theta_{\mathcal{R}}} \quad & \mathbb{E}_{\mathbf{x}\sim \mathcal{D}} \mathbb{E}_{\pi} \mathrm{S}(\mathcal{R}(\mathbf{y}_T), \mathbf{x}) \\
    \text{s.t.} \quad & \mathbf{y}_t = \mathbf{M}_t \odot \mathcal{F}(\mathbf{x}), \quad t = 0, 1, \cdots, T, \\
    & a_t \sim \pi(\cdot \mid \mathbf{y}_t; \theta_\pi), \quad t = 0, 1, \cdots, T-1, \\
    & \mathbf{M}_{t+1} = \mathds{1}(\mathbf{M}_t + \mathbf{M}^{a_t}), \quad t = 0, 1, \cdots, T-1, 
\end{split}
\end{equation}
and we denote the objective function $\mathbb{E}_{\pi} \mathrm{S}(\mathcal{R}(\mathbf{y}_T), \mathbf{x})$ in \eref{equ:sparse-optim-constrain} as $J_T^{\text{sparse}}(\mathbf{x}; \pi, \mathcal{R})$. 

Solving the sparse-reward joint optimization problem yields a reconstructor and a sampler that are at least as strong as those obtained from solving the dense-reward joint optimization problem.

\begin{theorem}[joint optimization problem]
The optimal values of optimization problems \eref{equ:dense-optim-constrain} and \eref{equ:sparse-optim-constrain} satisfy
\begin{equation}
\sup_{\substack{\pi \in C(\mathbb{R}^{N\times N}, \Delta^N) \\ \mathcal{R} \in C(\mathbb{C}^{N\times N}, \mathbb{R}^{N\times N})}} \mathbb{E}_{\mathbf{x} \sim \mathcal{D}} J_T^{\text{dense}}(\mathbf{x})
\leq \sup_{\substack{\pi \in C(\mathbb{C}^{N\times N}, \Delta^N)\\ \mathcal{R} \in C(\mathbb{C}^{N\times N}, \mathbb{R}^{N\times N})}} \mathbb{E}_{\mathbf{x} \sim \mathcal{D}} J_T^{\text{sparse}}(\mathbf{x}).
\end{equation}
\label{theorem:non-deterministic}
\end{theorem}
\begin{proof}
    See \ref{appen-subsec:proof-non-deterministic}. 
\end{proof}

\textbf{Remark.} 
In our optimization problems \eref{equ:dense-optim-constrain} and \eref{equ:sparse-optim-constrain}, the sampler and reconstructor are parameterized by $\theta_{\pi}$ and $\theta_{\mathcal{R}}$, reflecting their potential practical implementation as neural networks with learnable parameters. However, in theorem \ref{theorem:non-deterministic}, we place them within the broader continuous function space, which provides analytical convenience for theoretical explorations. This remark also applies to the following theorem \ref{theorem:improved-dynamic}. 

Theorem \ref{theorem:non-deterministic} gives us the confidence to switch from solving the dense-reward joint optimization problem defined in equation \eref{equ:dense-optim-constrain} to solving the sparse-reward joint optimization problem defined in equation \eref{equ:sparse-optim-constrain}.


Next, we can demonstrate that the sparse-reward POMDP is more computationally efficient than the dense-reward POMDP. The inference process of the dense-reward POMDP can be represented as follows: 
$$
\pstr[][-15pt]{
    \nd (y0){\mathbf{y}_0}
    \nd \xrightarrow{\mathcal{R}} \mathbf{x}_0 \xrightarrow[\mathbf{y}]{\pi} (y1){\mathbf{y}_1}
    \nd \xrightarrow{\mathcal{R}} \mathbf{x}_1 \xrightarrow[\mathbf{y}]{\pi} (y2){\mathbf{y}_2}
    \nd \cdots (yT){\mathbf{y}_T} 
\arrow{y0}{y1}{30}{}{black}{}
\arrow{y1}{y2}{30}{}{black}{}
}
    \xrightarrow{\mathcal{R}} \mathbf{x}_T.
$$
On the other hand, the proposed sparse-reward POMDP can be represented as:
$$
\pstr[][-15pt]{
    \nd (y0){\mathbf{y}_0}
    \nd \xrightarrow[\mathbf{y}]{\pi} (y1){\mathbf{y}_1}
    \nd \xrightarrow[\mathbf{y}]{\pi} (y2){\mathbf{y}_2}
    \nd \cdots (yT){\mathbf{y}_T} 
\arrow{y0}{y1}{30}{}{black}{}
\arrow{y1}{y2}{30}{}{black}{}
}
    \xrightarrow{\mathcal{R}} \mathbf{x}_T.
$$
It can be observed that the sparse-reward POMDP does not require reconstruction at the sampling stage. Thus, the sparse-reward POMDP is computationally more efficient. 

Then, we provide a theoretical explanation of how sparse-reward POMDP avoids the distributional mismatch in dense-reward POMDP. In the sparse-reward POMDP, the reconstructor only operates on the terminal state to obtain the reward, and a terminal heuristic sampling policy $\pi_T^{\text{h}}$ is used to pre-train the reconstructor $\mathcal{R}(\cdot)$. We can learn samplers with the proposed sparse-reward POMDP by solving the suboptimization problem of \eref{equ:sparse-optim-constrain} with a fixed pre-trained reconstruction model $\mathcal{R}(\cdot)$:
\begin{equation}
\max_{\theta_{\pi}} \quad \mathbb{E}_{\mathbf{x} \sim \mathcal{D}} J_T^{\text{sparse}}(\mathbf{x}; \pi(\cdot; \theta_\pi), \mathcal{R}). 
\label{equ:sparse-suboptim}
\end{equation}
The following theorem ensures that the sparse-reward suboptimization problem \eref{equ:sparse-suboptim} produces a sampler that is no weaker than the dense-reward one \eref{equ:dense-suboptim}, thus supporting the claim in the previous section.

\begin{theorem}[distributional mismatch]
Let the terminal sampling policy $\pi_{T+\| \mathbf{M}_0 \|_\infty}^{\text{h}}$ satisfies $\pi_{T+\| \mathbf{M}_0 \|_\infty}^{\text{h}}(\mathbf{a}) > 0$ for all binary column vectors $\mathbf{a}$ satisfying $\Vert \mathbf{a} \Vert_1 = T+\| \mathbf{M}_0 \|_\infty$. Suppose that there exits a reconstructor $\mathcal{R}^{\text{sparse}}$ satisfying
\begin{equation} \label{equ:R-sparse}
    \mathcal{R}^{\text{sparse}} \in \mathop{\mathrm{argmax}}\limits_{\mathcal{R} \in C(\mathbb{C}^{N\times N}, \mathbb{R}^{N\times N})} \mathbb{E}_{\mathbf{x}\sim\mathcal{D}} \mathbb{E}_{\mathbf{a}\sim\pi_{T+\| \mathbf{M}_0 \|_\infty}^{\text{h}}} \mathrm{S}(\mathcal{R}(\mathbf{M}^{\mathbf{a}} \odot \mathcal{F}(\mathbf{x})), \mathbf{x}). 
\end{equation}
Then, for any continuous reconstructor $\mathcal{R}^{\text{dense}}$, we have
\begin{equation} \label{equ:improved-dynamic}
    \sup_{\pi \in C(\mathbb{R}^{N\times N}, \Delta^N)} \mathbb{E}_{\mathbf{x} \sim \mathcal{D}} J_T^{\text{dense}}(\mathbf{x}; \pi, \mathcal{R}^{\text{dense}}) \leq \sup_{\pi \in C(\mathbb{C}^{N\times N}, \Delta^N)} \mathbb{E}_{\mathbf{x} \sim \mathcal{D}} J_T^{\text{sparse}}(\mathbf{x}; \pi, \mathcal{R}^{\text{sparse}}). 
\end{equation}
\label{theorem:improved-dynamic}
\end{theorem}
\begin{proof}
    See \ref{appen-subsec:dynamic}. 
\end{proof} 

\textbf{Remark 1}. Theorem \ref{theorem:improved-dynamic} demonstrates that the optimal value of \eref{equ:dense-suboptim} is no greater than that of \eref{equ:sparse-suboptim}. Therefore, compared to pre-training $\mathcal{R}^{\text{sparse}}$ with the terminal heuristic sampling policy $\pi_T^{\text{h}}$ for the sparse-reward POMDP, pre-training with any other heuristic sampling policies for the dense-reward POMDP yields no benefit. In particular, previous dynamic sampling methods have utilized a mixture of heuristic sampling policy $\pi_\text{mix}^{\text{h}} = \sum_{t = \| \mathbf{M}_0 \|_\infty}^{\| \mathbf{M}_0 \|_\infty + T} c_t \pi_t^{\text{h}}$ for pre-training reconstructors with the dense-reward POMDP, where $\sum c_t = 1$. Furthermore, we will demonstrate empirically that using $\pi_\text{mix}^{\text{h}}$ for pre-training is generally less effective than using $\pi_T^{\text{h}}$, which we refer to as a distributional mismatch.

\textbf{Remark 2}. Assumption `$\pi_{T+\| \mathbf{M}_0 \|_\infty}^{\text{h}}(\mathbf{a}) > 0$ for all binary column vectors $\mathbf{a}$ satisfying $\Vert \mathbf{a} \Vert_1 = T+\| \mathbf{M}_0 \|_\infty$' is necessary and easily attainable in experimental settings. The conclusion does not hold without this assumption, as demonstrated by a counter-example in \ref{appen-subsec:dynamic}.

Apart from improving inference efficiency and resolving the distributional mismatch, the proposed sparse-reward POMDP also offers an additional advantage. It decouples the sampling and reconstruction processes, which means that the sampling trajectories are no longer reliant on the reconstructor, but only on the sampler and images. This decoupling provides a viable solution for tackling the joint optimization problem of sparse-reward \eref{equ:sparse-optim-constrain}.

\subsection{Dynamic Sampling Training Framework}
\label{subsec:dynamic-sampling}
In this subsection, we present a dynamic sampling training framework that utilizes a pretrained fixed reconstructor to solve the sparse-reward suboptimization problem \eref{equ:sparse-suboptim}. Our experimental results in \sref{sec:exper} demonstrate that this algorithm outperforms previous RL-based methods that rely on dense-reward POMDP, thus providing further evidence that eliminating distributional mismatch can improve reconstruction performance.

We begin by pre-training the reconstructor using a terminal heuristic sampling policy $\pi_T^{\text{h}}$. Subsequently, we employ the proposed sparse-reward POMDP through RL to train the sampler with the fixed reconstructor. Mathematically, we first solve for $\theta_{\mathcal{R}}$ by maximizing the expected reconstruction loss $\mathrm{S}$ of the reconstructor:
\begin{equation}
    \theta_{\mathcal{R}} = \mathop{\mathrm{argmax}}\limits_{\theta_\mathcal{R}} \ \mathbb{E}_{\mathbf{x} \sim \mathcal{D}}\mathbb{E}_{\mathbf{a}\sim\pi_T^{\text{h}}} \left[ \mathrm{S} (\mathcal{R}(\mathbf{M}^\mathbf{a} \odot \mathcal{F}(\mathbf{x}); \theta_{\mathcal{R}}), \mathbf{x}) \right]. 
    \label{equ:dynamic-suboptim1}
\end{equation}
and then solve for $\theta_\pi$ by maximizing the expected cumulative reward $J_T^{\text{sparse}}$ under the sparse-reward POMDP framework:
\begin{equation}
    \theta_\pi = \mathop{\mathrm{argmax}}\limits_{\theta_\pi} \ \mathbb{E}_{\mathbf{x} \sim \mathcal{D}} J_T^{\text{sparse}}(\mathbf{x}; \pi(\cdot; \theta_{\pi}), \mathcal{R}(\cdot; \theta_{\mathcal{R}})). 
    \label{equ:dynamic-suboptim2}
\end{equation}

To solve the optimization problem \eref{equ:dynamic-suboptim1}, we use a standard DL-based MRI reconstruction approach and train the reconstructor using the Adam optimizer \cite{kingma2014adam}.

To solve the optimization problem \eref{equ:dynamic-suboptim2}, we leverage the proposed sparse-reward POMDP through RL. Specifically, we adopt the Actor-to-Critic (A2C) algorithm \cite{mnih2016asynchronous}, an efficient deep RL method to optimize the policy $\pi(\cdot ; \theta_\pi)$. 

We refer to the entire training process as L2S which stands for Learning to Sample. We summarize this process in algorithm \ref{alg:dynamic-framework} and illustrated it in \fref{subfig:framework-dynamic}. 

\begin{figure}
    \centering
    \subfigure[]{
    \centering
    \includegraphics[width=0.3\columnwidth]{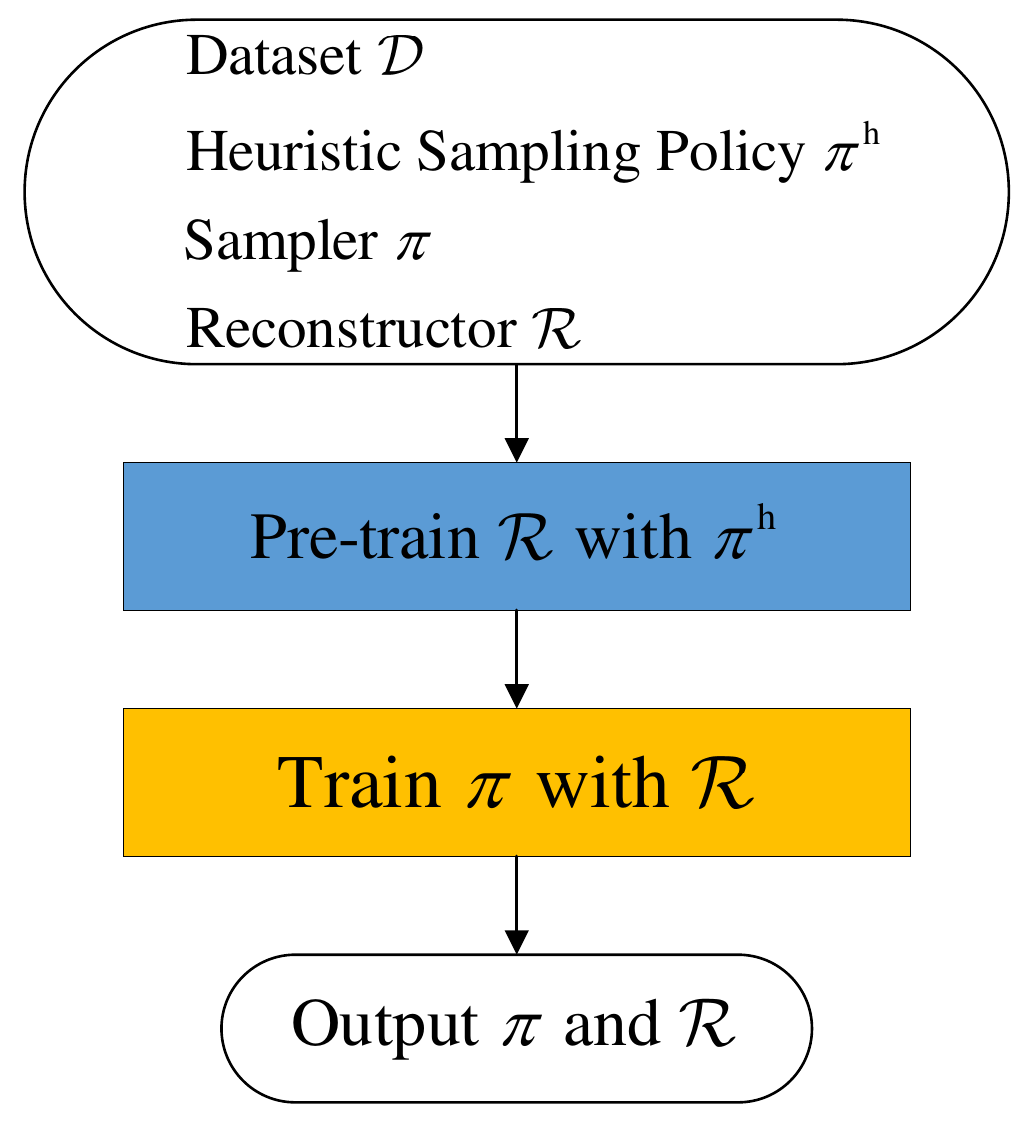}
    \label{subfig:framework-dynamic}
    }
    \subfigure[]{
    \centering
    \includegraphics[width=0.3\columnwidth]{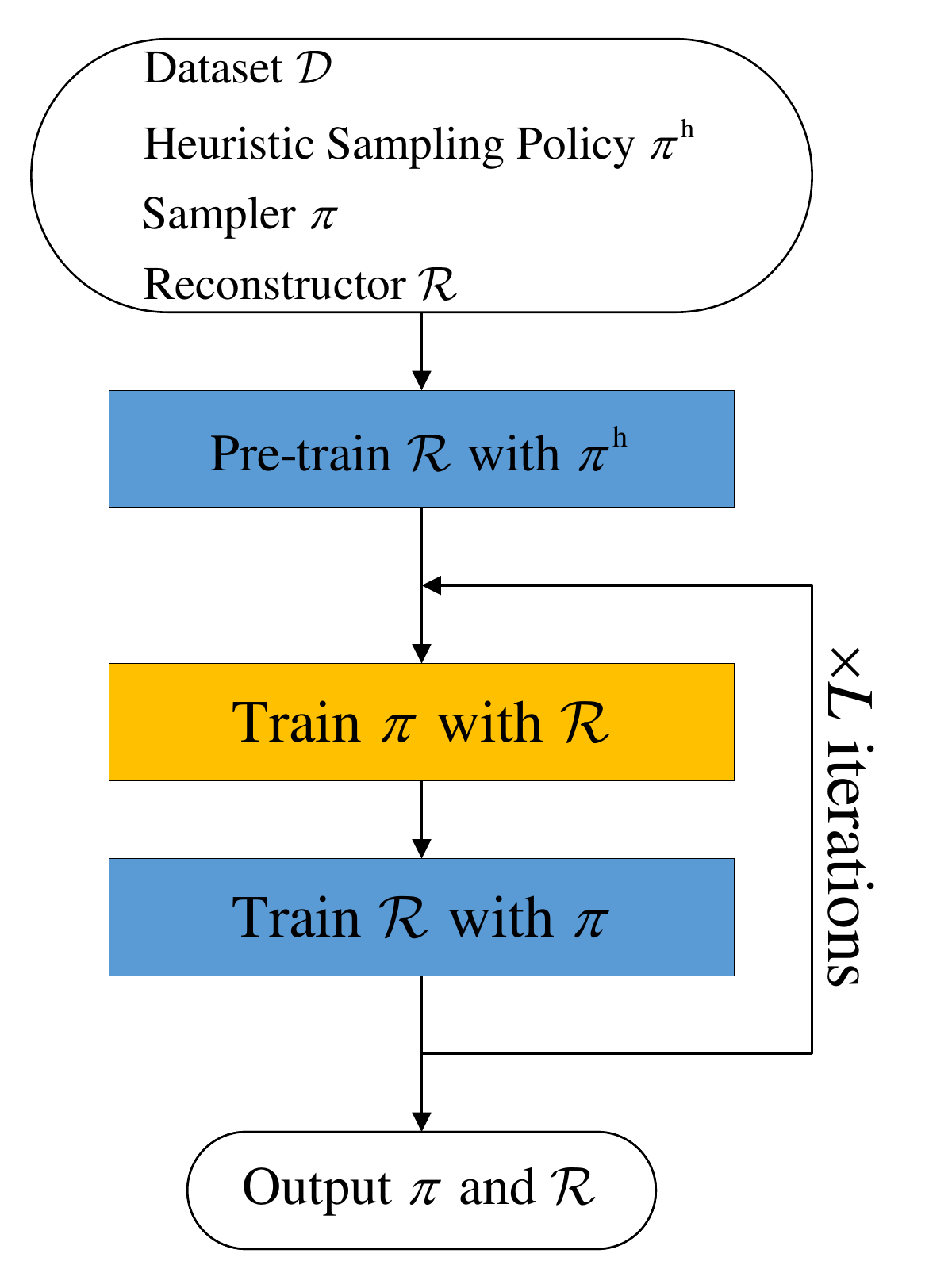}
    \label{subfig:framework-alternate}
    }
    \caption{Overview of (a) the training framework of dynamic sampling with fixed reconstructors, and (b) the alternating training framework. }
    \label{fig:framework}
\end{figure}

\begin{algorithm}[t]
\caption{Learning to Sample (L2S)}
\label{alg:dynamic-framework}
\begin{algorithmic}
\REQUIRE a sampler $\pi$, a reconstructor $\mathcal{R}$, a heuristic sampling policy $\pi^{\text{h}}$, an MRI dataset $\mathcal{D}$, and the maximum number of BP iterations $n$
\FOR{$i = 1 \ \text{to} \ n$}   
\STATE sample a batch $\{\mathbf{x}^{(b)}\}_{b=1}^B \sim \mathcal{D}$ and binary mask vectors $\mathbf{a}^{(b)} \sim \pi^{\text{h}}$
\STATE compute $\mathcal{L}(\theta_\mathcal{R}) = - \frac{1}{B} \sum_{b=1}^B \mathrm{S}(\mathcal{R}(\mathbf{M}^{\mathbf{a}^{(b)}}\mathcal{F}(\mathbf{x}^{(b)}); \theta_\mathcal{R}), \mathbf{x}^{(b)})$
\STATE optimize $\theta_\mathcal{R}$ with the loss $\mathcal{L}$ by Adam
\ENDFOR
\STATE build environment (sparse-reward POMDP)
\STATE train sampler $\pi$ with the environment by A2C
\STATE Output sampler $\pi$ and reconstructor $\mathcal{R}$
\end{algorithmic}
\end{algorithm}

\textbf{Remark.} While we still require a heuristic policy $\pi^{\text{h}}$, it is only used to initialize the reconstructor at the beginning and does not rely on human knowledge or design. Even random initialization is sufficient to provide a reasonable starting point before training (see \sref{sec:exper}). The core of our training process involves RL-based optimization of $\pi$ in a data-driven way, without restricting to human-designed heuristics. This remark also applies to algorithm \ref{alg:alter-framework}.

\subsection{Alternating Training Framework}
\label{subsec:alternate}
In this subsection, we present an alternating training framework to solve the sparse-reward joint optimization problem \eref{equ:sparse-optim-constrain}. By jointly optimizing the sampler and reconstructor, this algorithm eliminates \textit{training mismatch}. Mathematically, the optimal values of \eref{equ:sparse-optim-constrain} and \eref{equ:sparse-suboptim} satisfy the following inequality: 
$$
\sup_{\theta_\pi} \ \mathbb{E}_{\mathbf{x} \sim \mathcal{D}} J_T^{\text{sparse}}(\mathbf{x}; \pi(\cdot; \theta_\pi), \mathcal{R}) \leq \sup_{\theta_\pi, \theta_\mathcal{R}} \ \mathbb{E}_{\mathbf{x} \sim \mathcal{D}} J_T^{\text{sparse}}(\mathbf{x}; \pi(\cdot; \theta_\pi), \mathcal{R}(\cdot; \theta_\mathcal{R})).
$$

We begin by pre-training the reconstructor using a terminal heuristic sampling policy $\pi_T^{\text{h}}$. During training, we alternately train the sampler with the fixed reconstructor and train the reconstructor with the fixed sampler. Mathematically, we solve the following optimization problems in sequence:
\begin{equation}
    \theta_{\mathcal{R}}^{(0)} = \mathop{\mathrm{argmax}}\limits_{\theta_\mathcal{R}} \ \mathbb{E}_{\mathbf{x} \sim \mathcal{D}}\mathbb{E}_{\mathbf{a}\sim\pi_T^{\text{h}}} \left[ \mathrm{S} (\mathcal{R}(\mathbf{M}^\mathbf{a} \odot \mathcal{F}(\mathbf{x}); \theta_{\mathcal{R}}), \mathbf{x}) \right], 
    \label{equ:alter-suboptim1}
\end{equation}
\begin{equation}
    \theta_\pi^{(l)} = \mathop{\mathrm{argmax}}\limits_{\theta_\pi} \ \mathbb{E}_{\mathbf{x} \sim \mathcal{D}} J_T^{\text{sparse}}(\mathbf{x}; \pi(\cdot; \theta_{\pi}), \mathcal{R}(\cdot; \theta_{\mathcal{R}}^{(l-1)})),
    \label{equ:alter-suboptim2}
\end{equation}
\begin{equation}
    \theta_{\mathcal{R}}^{(l)} = \mathop{\mathrm{argmax}}\limits_{\theta_\mathcal{R}} \ \mathbb{E}_{\mathbf{x} \sim \mathcal{D}} J_T^{\text{sparse}}(\mathbf{x}; \pi(\cdot; \theta_{\pi}^{(l)}), \mathcal{R}(\cdot; \theta_{\mathcal{R}})),
    \label{equ:alter-suboptim3}
\end{equation}
where $L$ is the number of alternation and $l=1,2,\cdots,L$. 

Optimization problems \eref{equ:alter-suboptim1} and \eref{equ:alter-suboptim2} are identical to \eref{equ:dynamic-suboptim1} and \eref{equ:dynamic-suboptim2} respectively. We employ the same methods to solve them.

Thanks to the proposed sparse-reward POMDP, which separates sampling and reconstruction, we can solve optimization problem \eref{equ:alter-suboptim3} using BP. The following proposition enables training the reconstructor with a fixed sampler using gradient-based techniques such as the Adam optimizer.
\begin{proposition}
The derivative of $J_T^{\text{sparse}}$ w.r.t. $\theta_{\mathcal{R}}$ is
\begin{equation}
    \nabla_{\theta_{\mathcal{R}}} J_T^{\text{sparse}}(\mathbf{x}) = \mathbb{E}_{\{a_t\}_{t=0}^{T-1} \sim \pi} \left[ \nabla_{\theta_{\mathcal{R}}} \mathrm{S}(\mathcal{R}(\mathbf{y}_T; \theta_\mathcal{R}), \mathbf{x} \right], 
\label{equ:sparse-derivative}
\end{equation}
where $\{a_t\}_{t=0}^{T-1} \sim \pi$ means a sequential acquisition according to \eref{equ:sparse-optim-constrain} and $\mathbf{y}_T = \\ \mathbf{y}_T (\mathbf{M}_0, \{a_t\}_{t=0}^{T-1}) = (\mathbf{M}_0 + \sum_{t=0}^{T-1} \mathbf{M}^{a_t}) \odot \mathbf{y}$. 
\label{prop:sparse-derivative}
\end{proposition}
\begin{proof}
    See \ref{appen-subsec:proof-sparse-derivative}. 
\end{proof}

This proposition states that we can compute the gradient of $J_T^{\text{sparse}}$ w.r.t $\theta_\mathcal{R}$. Specifically, we sample a sampling trajectory $\{a_t\}_{t=0}^{T-1}$ from the ground truth $\mathbf{x}$ using the learned policy $\pi$ to get $\mathbf{y}_T$, then reconstruct $\mathbf{x}_T$ by applying the reconstruction function $\mathcal{R}$ to $\mathbf{y}_T$, and finally compute the gradient of the similarity between the reconstructed image and the ground truth w.r.t $\theta_\mathcal{R}$. Furthermore, we run multiple trajectories to compute the average gradient, and optimize \eref{equ:alter-suboptim3} by Adam optimizer. 

\begin{algorithm}[t]
\caption{Learning to Sample and Reconstruct (L2SR)}
\label{alg:alter-framework}
\begin{algorithmic}
\REQUIRE a sampler $\pi$, a reconstructor $\mathcal{R}$, a heuristic sampling policy $\pi^{\text{h}}$, an MRI dataset $\mathcal{D}$, the number of alternations $L$, and the maximum number of BP iterations $\{n^{(l)}\}_{l=0}^L$
\FOR{$i = 1 \ \text{to} \ n^{(0)}$}
\STATE sample a batch $\{\mathbf{x}^{(b)}\}_{b=1}^B \sim \mathcal{D}$ and binary mask vectors $\mathbf{a}^{(b)} \sim \pi^{\text{h}}$
\STATE compute $\mathcal{L}(\theta_\mathcal{R}) = - \frac{1}{B} \sum_{b=1}^B \mathrm{S}(\mathcal{R}(\mathbf{M}^{\mathbf{a}^{(b)}}\mathcal{F}(\mathbf{x}^{(b)}); \theta_\mathcal{R}), \mathbf{x}^{(b)})$
\STATE optimize $\theta_\mathcal{R}$ with the loss $\mathcal{L}$ by Adam
\ENDFOR
\FOR{$l = 1 \ \text{to} \ L$}
\STATE build environment (sparse-reward POMDP)
\STATE train sampler $\pi$ with the environment by A2C
\FOR{$i = 1 \ \text{to} \ n^{(l)}$}
\STATE sample a batch $\{\mathbf{x}^{(b)}\}_{b=1}^B \sim \mathcal{D}$ and its acquisition sequences $\mathbf{a}^{(b)} \sim \pi(\mathbf{x}^{(b)})$
\STATE compute $\mathcal{L}(\theta_\mathcal{R}) = - \frac{1}{B} \sum_{b=1}^B \mathrm{S}(\mathcal{R}(\mathbf{M}^{\mathbf{a}^{(b)}}\mathcal{F}(\mathbf{x}^{(b)}); \theta_\mathcal{R}), \mathbf{x}^{(b)})$
\STATE optimize $\theta_\mathcal{R}$ with the loss $\mathcal{L}$ by Adam
\ENDFOR
\ENDFOR
\STATE Output sampler $\pi$ and reconstructor $\mathcal{R}$
\end{algorithmic}
\end{algorithm}

We refer to the entire training process as L2SR which stands for Learning to Sample and Reconstruct. We summarize this process in algorithm \ref{alg:alter-framework} and illustrated it in \fref{subfig:framework-alternate}. 

\textbf{Remark.} The proposed alternating training framework is not suitable for the dense-reward POMDP, since there is no equivalent derivative of $J_T^{\text{dense}}$ w.r.t $\theta_{\mathcal{R}}$ like the one presented in proposition \ref{prop:sparse-derivative}. We explain this limitation in \ref{appen-subsec:dense-derivative}.

\section{Experiments}
\label{sec:exper}

\begin{table}[]
\caption{Acceleration factors and initial accelerations factors of random sampling policies for pre-training reconstructors. We utilize $\mathcal{R}^{\text{dense}}$ for PG-MRI \cite{NEURIPS2020_daed2103} and Greedy Oracle, and $\mathcal{R}^{\text{sparse}}$ for Random and L2S. }
\label{table:heuristic-policy}
\begin{indented}
\lineup
\item[]
    \begin{tabular}{lll}
    \br
              & \multicolumn{2}{c}{\textbf{$\times 4$ acceleration}}                \\
              & \textbf{acceleration factor} & \textbf{initial acceleration factor} \\ \hline
$\mathcal{R}^{\text{dense}}$ & {[}4,4,4,6,6,8{]}            & {[}4,6,8,6,8,8{]}                    \\
$\mathcal{R}^{\text{sparse}}$ & 4                            & 8 (Base) or 32 (Long)                                    \\ 
\br \br
              & \multicolumn{2}{c}{\textbf{$\times 8$ acceleration}}                \\
              & \textbf{acceleration factor} & \textbf{initial acceleration factor} \\ \hline
$\mathcal{R}^{\text{dense}}$ & {[}8,8,8,12,12,16{]}         & {[}8,12,16,12,16,16{]}               \\
$\mathcal{R}^{\text{sparse}}$ & 8                            & 16 (Base) or 64 (Long)                                     \\
\br \br
              & \multicolumn{2}{c}{\textbf{$\times 16$ acceleration}}                \\
              & \textbf{acceleration factor} & \textbf{initial acceleration factor} \\ \hline
$\mathcal{R}^{\text{dense}}$ & {[}16,16,16,24,24,32{]}         & {[}16,16,16,24,24,32{]}               \\
$\mathcal{R}^{\text{sparse}}$ & 16                            & 32 (Base) or 128 (Long)                                     \\
\br
\end{tabular}
\end{indented}
\end{table}

\subsection{Implementation}
\label{subsec:exp-implementation}

\textbf{Setup}: We evaluate the proposed methods under two different settings. In the `\textbf{Fixed Reconstructor}' (abbreviated as `Fixed") setting, the reconstructor is pre-trained using a heuristic sampling policy and is subsequently utilized with our proposed \textbf{L2S} method. In contrast, the `\textbf{Joint Training}' (abbreviated as `Joint") setting involves the simultaneous training of both the sampler and the reconstructor, employing our novel \textbf{L2SR} framework. 

\textbf{Similarity Metric}: In our training process, we adopt the Structural Similarity Index (SSIM) \cite{1284395} as the image similarity metric $\mathrm{S}$. We recognize its potential limitations for medical imaging tasks due to its sensitivity to scale differences, shifts, and rotations. It reflects the broader challenge within medical imaging of aligning quantitative evaluation metrics closely with the qualitative `eyeball' assessments used by medical professionals. To address these concerns, we emphasize our framework's inherent adaptability, designed to accommodate a variety of metrics such as Peak Signal-to-Noise Ratio (PSNR), Wasserstein GAN (WGAN) distance \cite{arjovsky2017wasserstein, huang2019learning} or some downstream task metrics, depending on the particular application requirements. This adaptability ensures the framework's applicability across various medical imaging tasks without being limited by the choice of SSIM.

\textbf{Acquisition}: We define the \textit{acceleration factor} as $N / \Vert \mathbf{a} \Vert_1$ to measure the time overhead of sampling, denoted as $\times (N / \Vert \mathbf{a} \Vert_1)$-acceleration. Most sampling policies start by acquiring a certain number of columns at the central low-frequency region, where more information is concentrated, and then sample the rest of the k-space according to their unique strategies. We define the \textit{initial acceleration factor} as the ratio between $N$ and the number of sampling columns in the first state. We explore two types of initialization: \textit{Base-horizon}, where the initial acceleration factor is twice the acceleration factor, and \textit{Long-horizon}, where it is 8 times the acceleration factor.

\textbf{Heuristic Sampling Policy for Pre-training}: For the pre-training of reconstructors $\mathcal{R}^{\text{dense}}$ and $\mathcal{R}^{\text{sparse}}$, we employ a random sampling policy as the heuristic approach. The specific acceleration factors and initial acceleration factors used in this process are detailed in \tref{table:heuristic-policy}. It is important to note that the selection of heuristic sampling policies for pre-training $\mathcal{R}^{\text{sparse}}$ aligns with the underlying assumptions detailed in theorem \ref{theorem:improved-dynamic}. This consistency ensures that our approach is grounded in a well-defined theoretical framework.

\textbf{Dataset}: We utilize the single-coil knee dataset and the multi-coil brain dataset from the commonly-used fastMRI dataset \cite{zbontar2018fastmri}. The data preprocessing follows the protocol established in \cite{NEURIPS2020_daed2103}: We partition the dataset, reserving 20\% of it for test and the rest for training. For the single-coil knee dataset, we derive the ground truth image $\mathbf{x}$ by cropping the original to the central $128\times128$ region, using half the available volumes, and removing the outer slices of each volume, resulting in 6959 training slices, 1779 validation slices, and 1715 test slices. For the multi-coil brain dataset, we obtain the ground truth image $\mathbf{x}$ by cropping to the central $256 \times 256$ region, using one-fifth of the available volumes, and assembling a dataset of 11312 training slices, 4372 validation slices, and 2832 test slices.

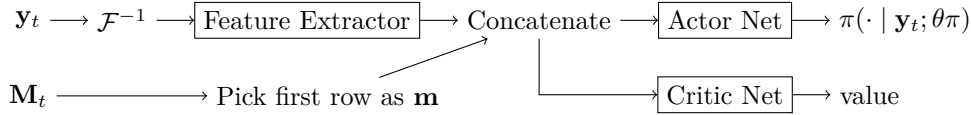
\begin{figure}
    \centering
    \begin{tikzpicture}
        \node (inputy) {$\mathbf{y}_t$};
        \node[below=0.5cm of inputy] (inputM) {$\mathbf{M}_t$};
        \node[right=0.5cm of inputy] (ifft) {$\mathcal{F}^{-1}$};
        \node[draw, right=0.5cm of ifft] (feature) {Feature Extractor};
        \node[right=0.5cm of feature] (concat) {Concatenate};
        \node[right=2cm of inputM] (pickm) {Pick first row as $\mathbf{m}$};
        \node[draw, right=0.5cm of concat] (actor) {Actor Net};
        \node[draw, below=0.5cm of actor] (critic) {Critic Net};
        \node[right=0.5cm of actor] (actorOut) {$\pi(\cdot \mid \mathbf{y}_t; \theta\pi)$};
        \node[right=0.5cm of critic] (criticOut) {value};

        \draw[->] (inputy) -- (ifft);
        \draw[->] (ifft) -- (feature);
        \draw[->] (feature) -- (concat);
        \draw[->] (inputM) -- (pickm);
        \draw[->] (pickm) -- (concat);
        \draw[->] (concat) -- (actor);
        \draw[->] (concat) |- (critic);
        \draw[->] (actor) -- (actorOut);
        \draw[->] (critic) -- (criticOut);
    \end{tikzpicture}
    \caption{Transversal view of the policy model comprising a 2D Inverse Fourier Transform, a feature extractor, and actor and critic neural networks. The learnable parts are enclosed in boxes. The feature extractor is a CNN-based neural network following the architecture from \cite{NEURIPS2020_daed2103}. The actor and critic nets are both constructed as fully connected neural networks.}
    \label{fig:policy-model}
\end{figure}

\textbf{Policy model}: The policy model's architecture is depicted in \fref{fig:policy-model}. It takes as input the binary mask matrix $\mathbf{M}_t$ and the observation $\mathbf{y}_t$, and processes them through a series of components including a 2D Inverse Fourier Transform, a feature extractor, and actor and critic networks. The final outputs are a discrete probability distribution $\pi(\cdot \mid \mathbf{y}_t; \theta\pi)$ over actions, generated by the actor net, and a scalar value representing the predicted value of the current state, produced by the critic net.

\textbf{Reconstruction model}: We employ an easy-to-implement standard U-Net \cite{ronneberger2015u} architecture with 16 channels, consisting of 8 blocks, as found in the fastMRI repository. The input to it is the zero-filled reconstruction image, represented by $|\mathcal{F}^{-1}(\mathbf{y}_t)|$, having the shape $(N, N, 1)$, and the output is the reconstruction image with the same shape. While we use Unet with 16 channels for our experiments, it's essential to note that our L2SR framework does not confine the choice of reconstruction models. The usage of the U-Net model serves as an illustrative example. 

\textbf{Computational Resources}: All experiments were conducted using a dedicated computational cluster equipped with NVIDIA Tesla V100 GPUs, each with 16 GB of memory. Our training and inference processes are performed on individual GPUs. 

\begin{table}[t]
\footnotesize
\centering
\caption{Main results: reconstruction results (mean and standard deviation) in terms of SSIM and PSNR values on the test dataset. For `Fixed (Reconstructor)', $+ \mathcal{R}$ means the heuristic sampling policy for pre-training. The best results (highest mean) for a specific acceleration factor and initial acceleration factor among the compared algorithms are shown in bold numbers. The best results for a specific acceleration factor are shown in blue numbers. }
\label{table:sparse-ssim}
\subtable[Knee dataset $\times 4$-acceleration]{
\resizebox{0.98\linewidth}{!}{%
\centering
\begin{tabular}{llllll}
\br
                          & & \multicolumn{2}{c}{\textbf{Base-horizon}}     & \multicolumn{2}{c}{\textbf{Long-horizon}}     \\
                          & & SSIM & PSNR & SSIM & PSNR \\ \midrule
\multicolumn{1}{l}{\multirow{4}{*}{\begin{tabular}[c]{@{}l@{}}Fixed\end{tabular}}} & \textbf{Random}+$\mathcal{R}^{\text{sparse}}$    & $0.7222 \pm 0.0405$ & $26.27 \pm 1.78$ & $0.715 \pm 0.0329$ & $25.16 \pm 1.78$ \\
\multicolumn{1}{l}{} & \textbf{PG-MRI}+$\mathcal{R}^{\text{dense}}$\cite{NEURIPS2020_daed2103}    & $0.7523 \pm 0.0375$ & $26.54 \pm 1.74$ & $0.7674 \pm 0.0299$ & $25.23 \pm 1.59$ \\
\multicolumn{1}{l}{} & \textbf{L2S}+$\mathcal{R}^{\text{sparse}}$ \textbf{(Ours)}    & $\bm{0.7543 \pm 0.0372}$ & $\bm{26.93 \pm 1.79}$ & $\bm{0.7838 \pm 0.0286}$ & $\bm{26.86 \pm 1.68}$ \\ \cline{2-6} 
\multicolumn{1}{l}{} & \textbf{Greedy Oracle}+$\mathcal{R}^{\text{dense}}$    & $0.7658 \pm 0.035$ & $26.74 \pm 1.75$ & $0.7837 \pm 0.0264$ & $25.79 \pm 1.7$ \\ \midrule \midrule
\multicolumn{1}{l}{\multirow{4}{*}{\begin{tabular}[c]{@{}l@{}}Joint\end{tabular}}} & \textbf{LOUPE}\cite{bahadir2019learning}    & $0.7243 \pm 0.0387$ & $26.2 \pm 1.75$ & $0.72 \pm 0.0314$ & $25.09 \pm 1.79$ \\
\multicolumn{1}{l}{} & \textbf{$\tau$-Step Seq}\cite{yin2021end}    & $0.7649 \pm 0.0408$ & $27.26 \pm 1.85$ & $0.8025 \pm 0.0315$ & $27.69 \pm 1.81$ \\
\multicolumn{1}{l}{} & \textbf{L2SR (Ours)}    & $\bm{0.7681 \pm 0.0416}$ & $\bm{27.38 \pm 1.9}$ & $\color{blue} \bm{0.8097 \pm 0.0333}$ & $\color{blue} \bm{28.15 \pm 1.83}$ \\
\br
\end{tabular}}}

\subtable[Knee dataset $\times 8$-acceleration]{
\resizebox{0.98\linewidth}{!}{%
\centering
\begin{tabular}{llllll}
\br
                          & & \multicolumn{2}{c}{\textbf{Base-horizon}}     & \multicolumn{2}{c}{\textbf{Long-horizon}}    \\
                          & & SSIM & PSNR & SSIM & PSNR \\ \midrule
\multicolumn{1}{l}{\multirow{4}{*}{\begin{tabular}[c]{@{}l@{}}Fixed\end{tabular}}} & \textbf{Random}+$\mathcal{R}^{\text{sparse}}$    & $0.6039 \pm 0.0492$ & $24.17 \pm 1.71$ & $0.5915 \pm 0.0442$ & $22.86 \pm 2.01$ \\
\multicolumn{1}{l}{} & \textbf{PG-MRI}+$\mathcal{R}^{\text{dense}}$\cite{NEURIPS2020_daed2103}    & $0.623 \pm 0.0504$ & $24.21 \pm 1.69$ & $0.633 \pm 0.0446$ & $23.51 \pm 1.74$ \\
\multicolumn{1}{l}{} & \textbf{L2S}+$\mathcal{R}^{\text{sparse}}$ \textbf{(Ours)}    & $\bm{0.6258 \pm 0.0505}$ & $\bm{24.65 \pm 1.72}$ & $\bm{0.6441 \pm 0.0453}$ & $\bm{24.28 \pm 1.71}$ \\ \cline{2-6} 
\multicolumn{1}{l}{} & \textbf{Greedy Oracle}+$\mathcal{R}^{\text{dense}}$    & $0.6376 \pm 0.0488$ & $24.44 \pm 1.71$ & $0.6509 \pm 0.043$ & $23.41 \pm 1.74$ \\ \midrule \midrule
\multicolumn{1}{l}{\multirow{4}{*}{\begin{tabular}[c]{@{}l@{}}Joint\end{tabular}}} & \textbf{LOUPE}\cite{bahadir2019learning}    & $0.6002 \pm 0.0474$ & $23.96 \pm 1.7$ & $0.5876 \pm 0.0413$ & $22.64 \pm 1.97$ \\
\multicolumn{1}{l}{} & \textbf{$\tau$-Step Seq}\cite{yin2021end}    & $\bm{0.6359 \pm 0.0572}$ & $\bm{25.07 \pm 1.73}$ &$0.6614 \pm 0.0508$ & $25.26 \pm 1.76$ \\
\multicolumn{1}{l}{} & \textbf{L2SR (Ours)}    & $0.6332 \pm 0.053$ & $24.98 \pm 1.74$ & $\color{blue} \bm{0.6712 \pm 0.0511}$ & $\color{blue} \bm{25.61 \pm 1.79}$ \\
\br
\end{tabular}}}

\subtable[Brain dataset $\times 8$-acceleration]{
\resizebox{0.98\linewidth}{!}{%
\centering
\begin{tabular}{llllll}
\br
                          & & \multicolumn{2}{c}{\textbf{Base-horizon}}     & \multicolumn{2}{c}{\textbf{Long-horizon}}     \\
                          & & PSNR & SSIM & PSNR & SSIM \\ \midrule
\multicolumn{1}{l}{\multirow{4}{*}{\begin{tabular}[c]{@{}l@{}}Fixed\end{tabular}}} & \textbf{Random}+$\mathcal{R}^{\text{sparse}}$   & $0.8643 \pm 0.053$ & $30.66 \pm 4.35$ & $0.8405 \pm 0.0668$ & $28.65 \pm 4.68$ \\
\multicolumn{1}{l}{} & \textbf{PG-MRI}+$\mathcal{R}^{\text{dense}}$\cite{NEURIPS2020_daed2103}    & $0.8771 \pm 0.0473$ & $31.37 \pm 4.18$ & $\bm{0.8766 \pm 0.0476}$ & $\bm{31.31 \pm 4.17}$ \\
\multicolumn{1}{l}{} & \textbf{L2S}+$\mathcal{R}^{\text{sparse}}$ \textbf{(Ours)}    & $\bm{0.8782 \pm 0.0475}$ & $\bm{31.54 \pm 4.17}$ & $0.8686 \pm 0.0529$ & $30.25 \pm 4.23$ \\ \cline{2-6} 
\multicolumn{1}{l}{} & \textbf{Greedy Oracle}+$\mathcal{R}^{\text{dense}}$    & $0.8822 \pm 0.0453$ & $31.64 \pm 4.15$ & $0.8771 \pm 0.0532$ & $31.17 \pm 4.31$ \\ \midrule \midrule
\multicolumn{1}{l}{\multirow{4}{*}{\begin{tabular}[c]{@{}l@{}}Joint\end{tabular}}} & \textbf{LOUPE}\cite{bahadir2019learning}    & $0.8553 \pm 0.0555$ & $30.15 \pm 4.32$ & $0.8211 \pm 0.0691$ & $27.34 \pm 3.99$ \\
\multicolumn{1}{l}{} & \textbf{$\tau$-Step Seq}\cite{yin2021end}    & $0.8862 \pm 0.041$ & $31.93 \pm 3.98$ & $0.8921 \pm 0.0397$ & $32.07 \pm 3.88$ \\
\multicolumn{1}{l}{} & \textbf{L2SR (Ours)}    & $\bm{0.8899 \pm 0.043}$ & $\bm{32.57 \pm 4.09}$ & $\color{blue} \bm{0.8969 \pm 0.0417}$ & $\color{blue} \bm{32.87 \pm 4.02}$ \\
\br
\end{tabular}}}

\subtable[Brain dataset $\times 16$-acceleration]{
\resizebox{0.98\linewidth}{!}{%
\centering
\begin{tabular}{llllll}
\br
                          & & \multicolumn{2}{c}{\textbf{Base-horizon}}     & \multicolumn{2}{c}{\textbf{Long-horizon}}    \\
                          & & SSIM & PSNR & SSIM & PSNR \\ \midrule
\multicolumn{1}{l}{\multirow{4}{*}{\begin{tabular}[c]{@{}l@{}}Fixed\end{tabular}}} & \textbf{Random}+$\mathcal{R}^{\text{sparse}}$   & $80.31 \pm 0.0825$ & $27.56 \pm 5.03$ & $0.7837 \pm 0.0921$ & $26.28 \pm 5.02$ \\
\multicolumn{1}{l}{} & \textbf{PG-MRI}+$\mathcal{R}^{\text{dense}}$\cite{NEURIPS2020_daed2103}    & $0.821 \pm 0.0722$ & $28.25 \pm 4.71$ & $\bm{0.8218 \pm 0.0721}$ & $\bm{28.36 \pm 4.68}$ \\
\multicolumn{1}{l}{} & \textbf{L2S}+$\mathcal{R}^{\text{sparse}}$ \textbf{(Ours)}    & $\bm{0.8263 \pm 0.0703}$ & $\bm{28.59 \pm 4.73}$ & $0.819 \pm 0.0738$ & $27.65 \pm 4.71$ \\ \cline{2-6} 
\multicolumn{1}{l}{} & \textbf{Greedy Oracle}+$\mathcal{R}^{\text{dense}}$   & $0.8312 \pm 0.0677$ & $28.75 \pm 4.64$ & $0.8251 \pm 0.0734$ & $28.36 \pm 4.86$ \\ \midrule \midrule
\multicolumn{1}{l}{\multirow{4}{*}{\begin{tabular}[c]{@{}l@{}}Joint\end{tabular}}} & \textbf{LOUPE}\cite{bahadir2019learning}    & $0.8003 \pm 0.0821$ & $27.55 \pm 4.977$ & $0.7855 \pm 0.0894$ & $26.29 \pm 4.748$ \\
\multicolumn{1}{l}{} & \textbf{$\tau$-Step Seq}\cite{yin2021end}    & $0.8314 \pm 0.0652$ & $29.05 \pm 4.636$ & $0.8384 \pm 0.0631$ & $29.34 \pm 4.59$ \\
\multicolumn{1}{l}{} & \textbf{L2SR (Ours)}    & $\bm{0.8404 \pm 0.066}$ & $\bm{29.39 \pm 4.69}$ & $\color{blue} \bm{0.8456 \pm 0.0652}$ & $\color{blue} \bm{29.37 \pm 4.74}$ \\
\br
\end{tabular}}}
\end{table}

\subsection{Algorithm Comparisons}
\label{subsec:alg-compare}
In the `Fixed Reconstructor' setting, we evaluate the L2S against three baseline methods: (1) Random: randomly selecting 1-d lines from a uniform distribution; (2) PG-MRI \cite{NEURIPS2020_daed2103}: solving the dense-reward POMDP by policy gradient (which is one of the state-of-the-art dynamic sampling methods); (3) Greedy Oracle: a one-step oracle policy that has access to ground truth at test time. We pre-train the dense-reward reconstructor $\mathcal{R}^{\text{dense}}$ with a mixture heuristic sampling policy for PG-MRI and Greedy Oracle, and pre-train the sparse-reward reconstructor $\mathcal{R}^{\text{sparse}}$ with a terminal heuristic sampling policy for Random and L2S.

In the `Joint Training' setting, we evaluate the L2SR against two baseline methods: (1) LOUPE \cite{bahadir2019learning}: jointly training a learnable non-sequential sampler and a reconstructor; (2) $\tau$-Step Seq. \cite{yin2021end}: an advanced end-to-end sequential sampling and reconstruction method, where $\tau$ indicates the number of sampling steps, typically set to either 2 or 4 (and we report the best results obtained with either 2 or 4).

Additional details regarding all methods can be found in \ref{appen-sec:implementation}.

\begin{figure}
    \centering
    \begin{minipage}{0.35\columnwidth}
    \subfigure[Knee dataset, $\times 4$ Base. ]{
    \centering
        \includegraphics[width=\columnwidth]{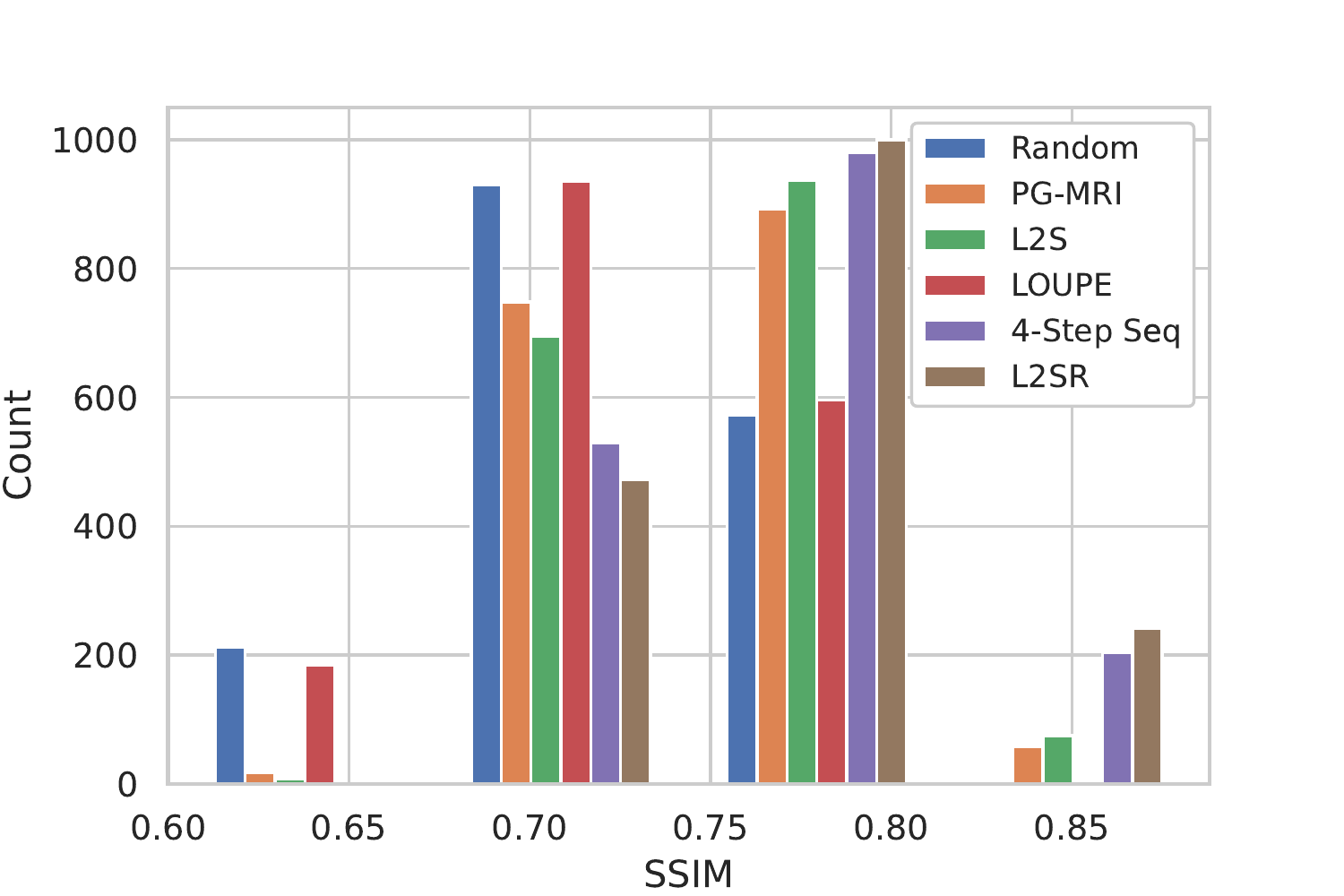}
    }
    \end{minipage}
    \begin{minipage}{0.35\columnwidth}
    \subfigure[Knee dataset, $\times 4$ Long. ]{
    \centering
        \includegraphics[width=\columnwidth]{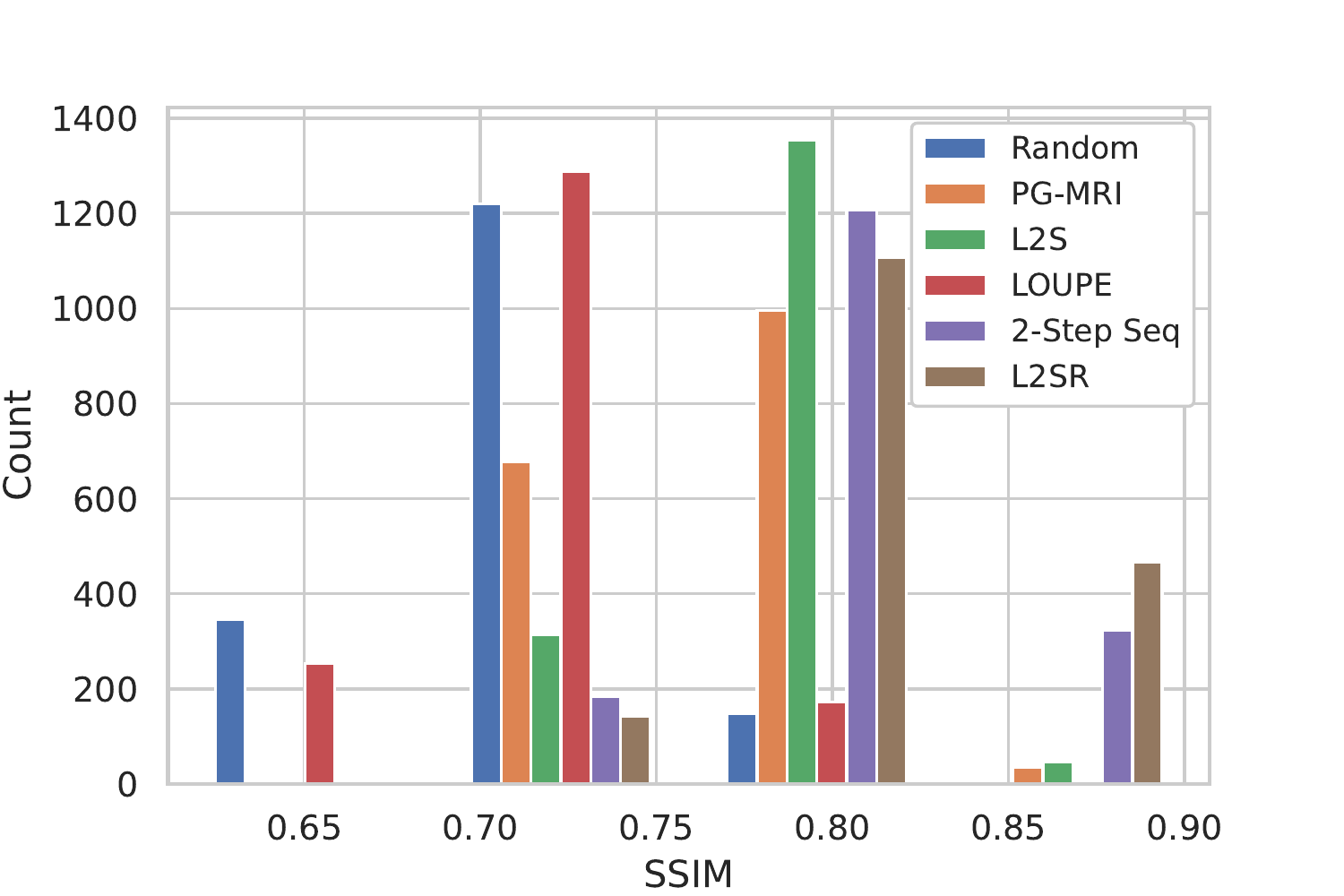}
    }
    \end{minipage}
    \begin{minipage}{0.35\columnwidth}
    \subfigure[Knee dataset, $\times 8$ Base. ]{
    \centering
        \includegraphics[width=\columnwidth]{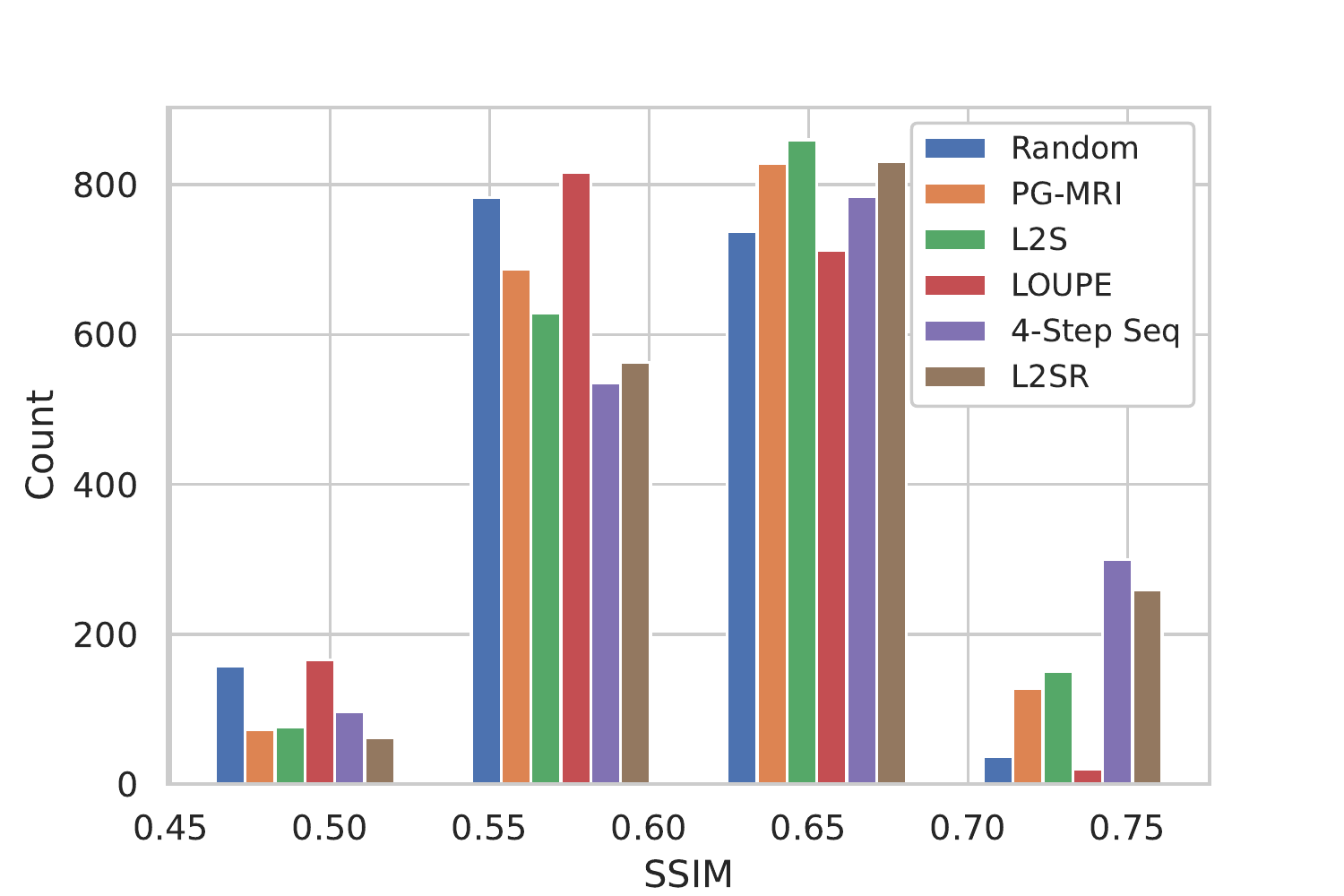}
    }
    \end{minipage}
    \begin{minipage}{0.35\columnwidth}
    \subfigure[Knee dataset, $\times 8$ Long. ]{
    \centering
        \includegraphics[width=\columnwidth]{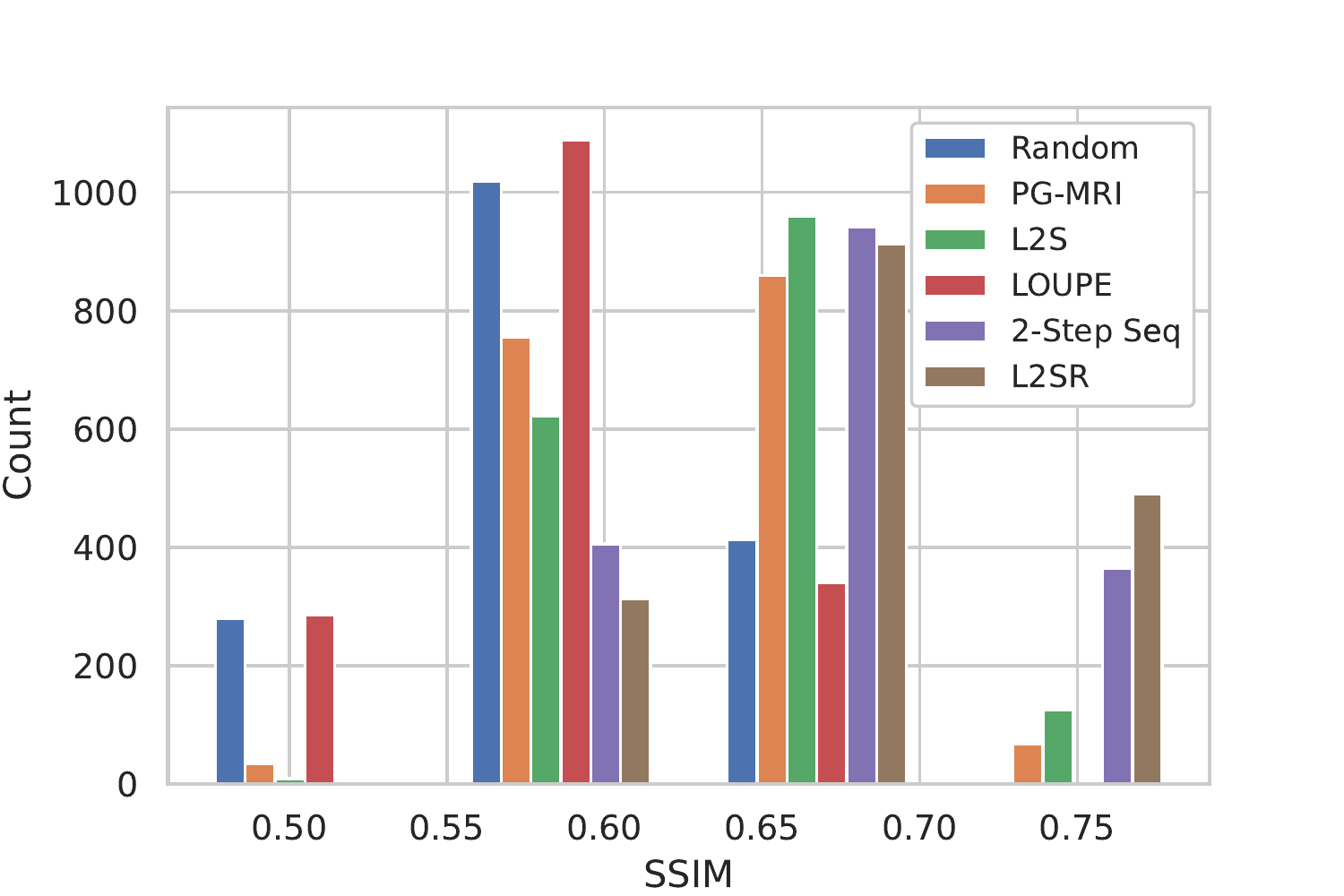}
    }
    \end{minipage}
    \begin{minipage}{0.35\columnwidth}
    \subfigure[Brain dataset, $\times 8$ Base. ]{
    \centering
        \includegraphics[width=\columnwidth]{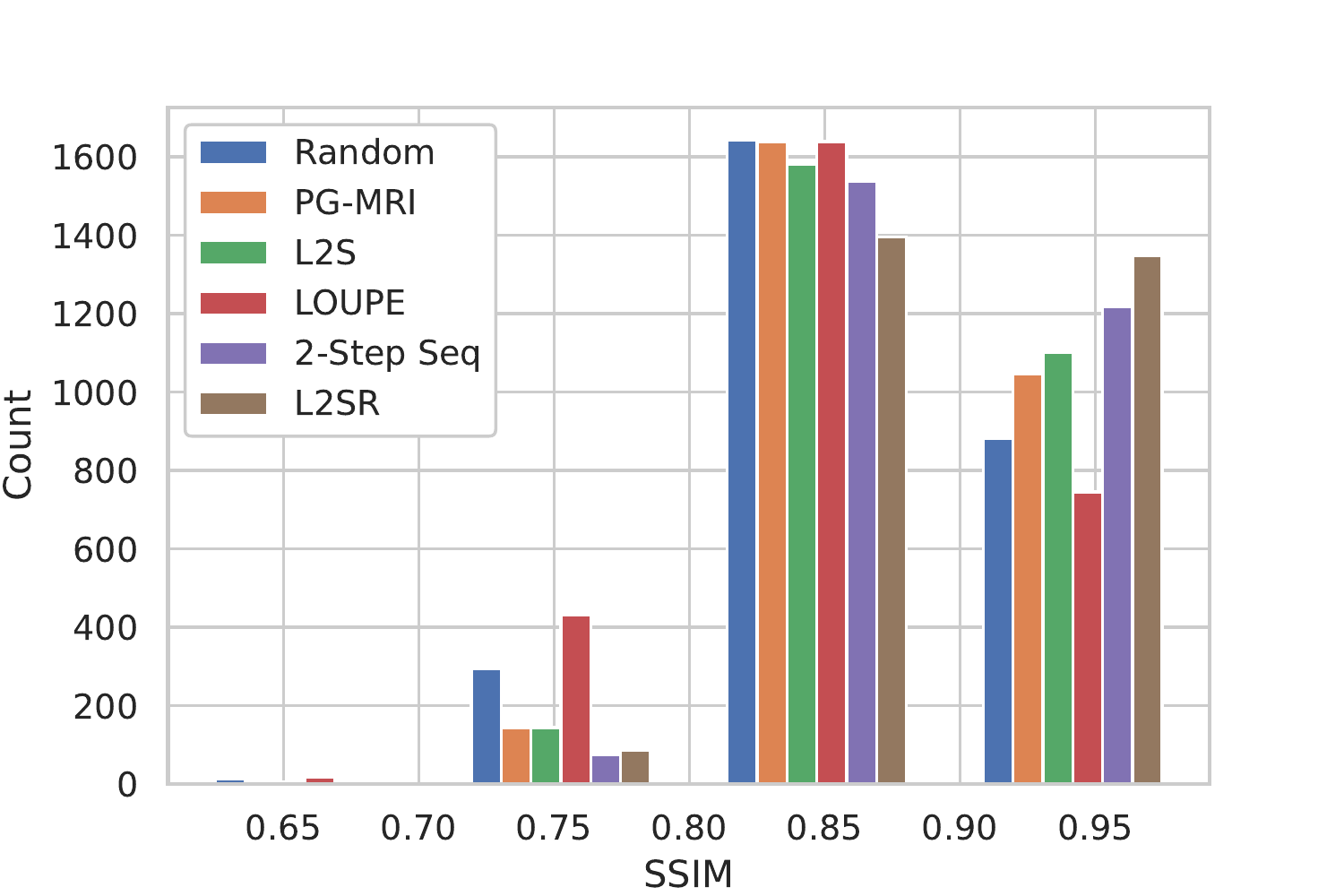}
    }
    \end{minipage}
    \begin{minipage}{0.35\columnwidth}
    \subfigure[Brain dataset, $\times 8$ Long. ]{
    \centering
        \includegraphics[width=\columnwidth]{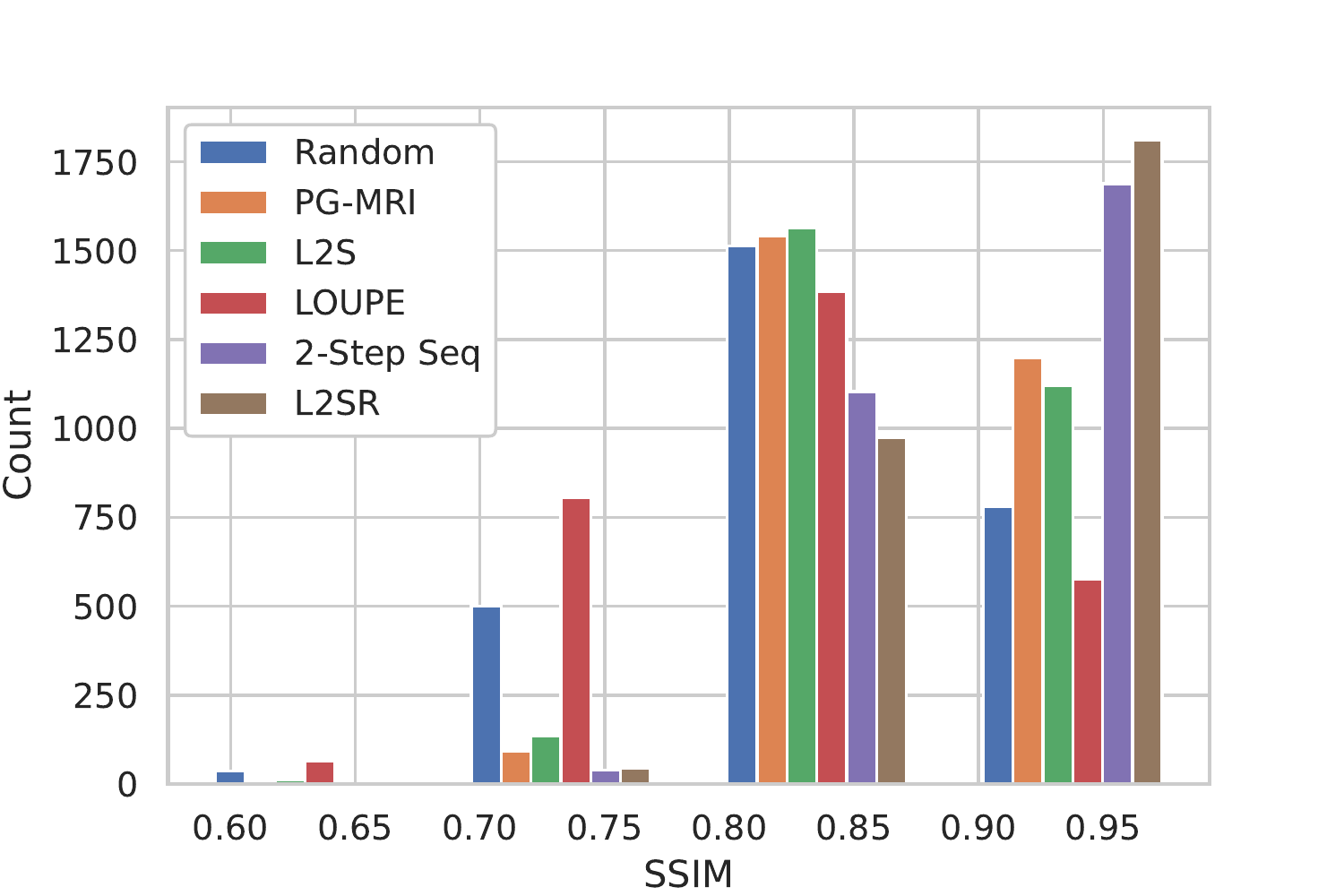}
    }
    \end{minipage}
    \begin{minipage}{0.35\columnwidth}
    \subfigure[Brain dataset, $\times 16$ Base. ]{
    \centering
        \includegraphics[width=\columnwidth]{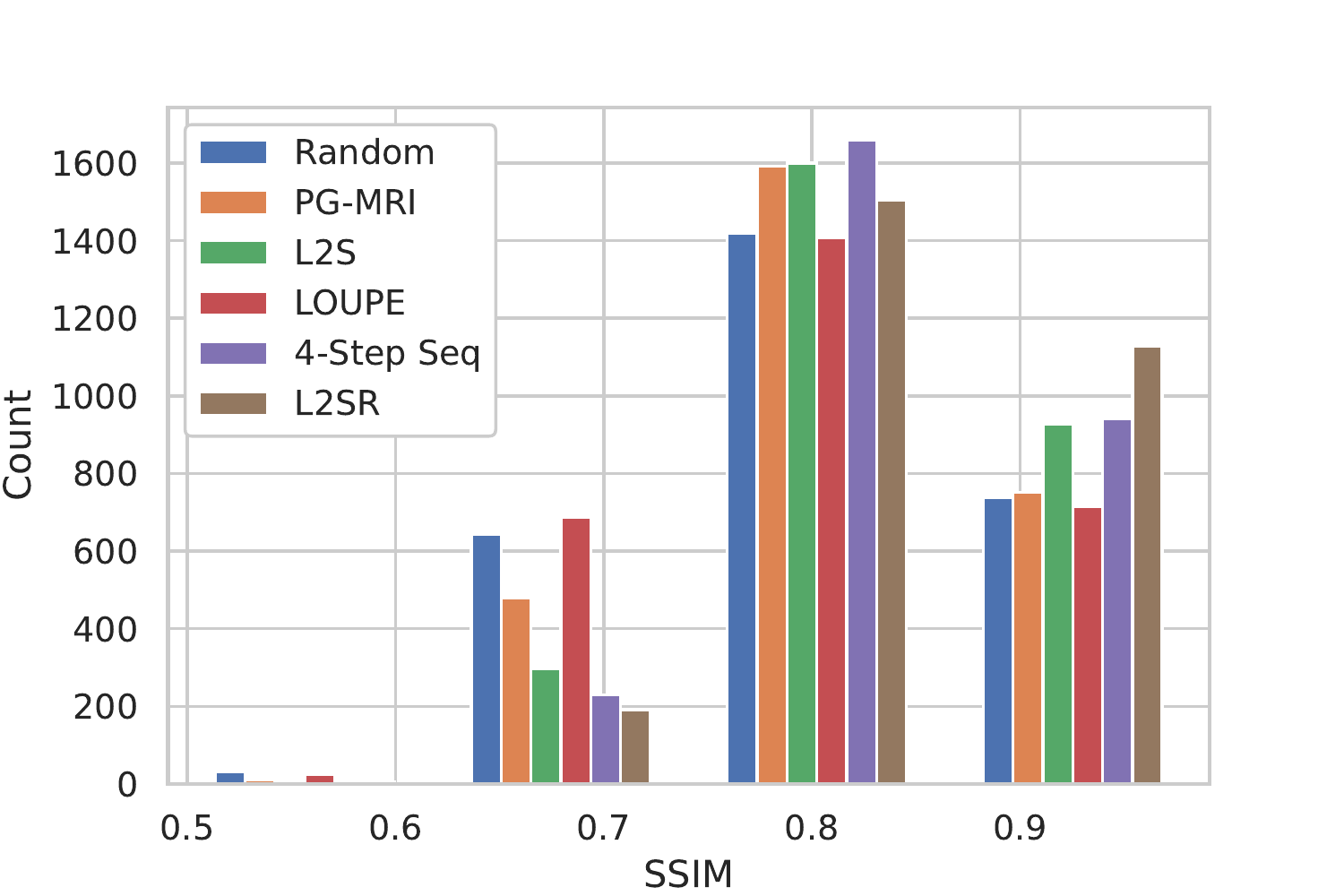}
    }
    \end{minipage}
    \begin{minipage}{0.35\columnwidth}
    \subfigure[Brain dataset, $\times 16$ Long. ]{
    \centering
        \includegraphics[width=\columnwidth]{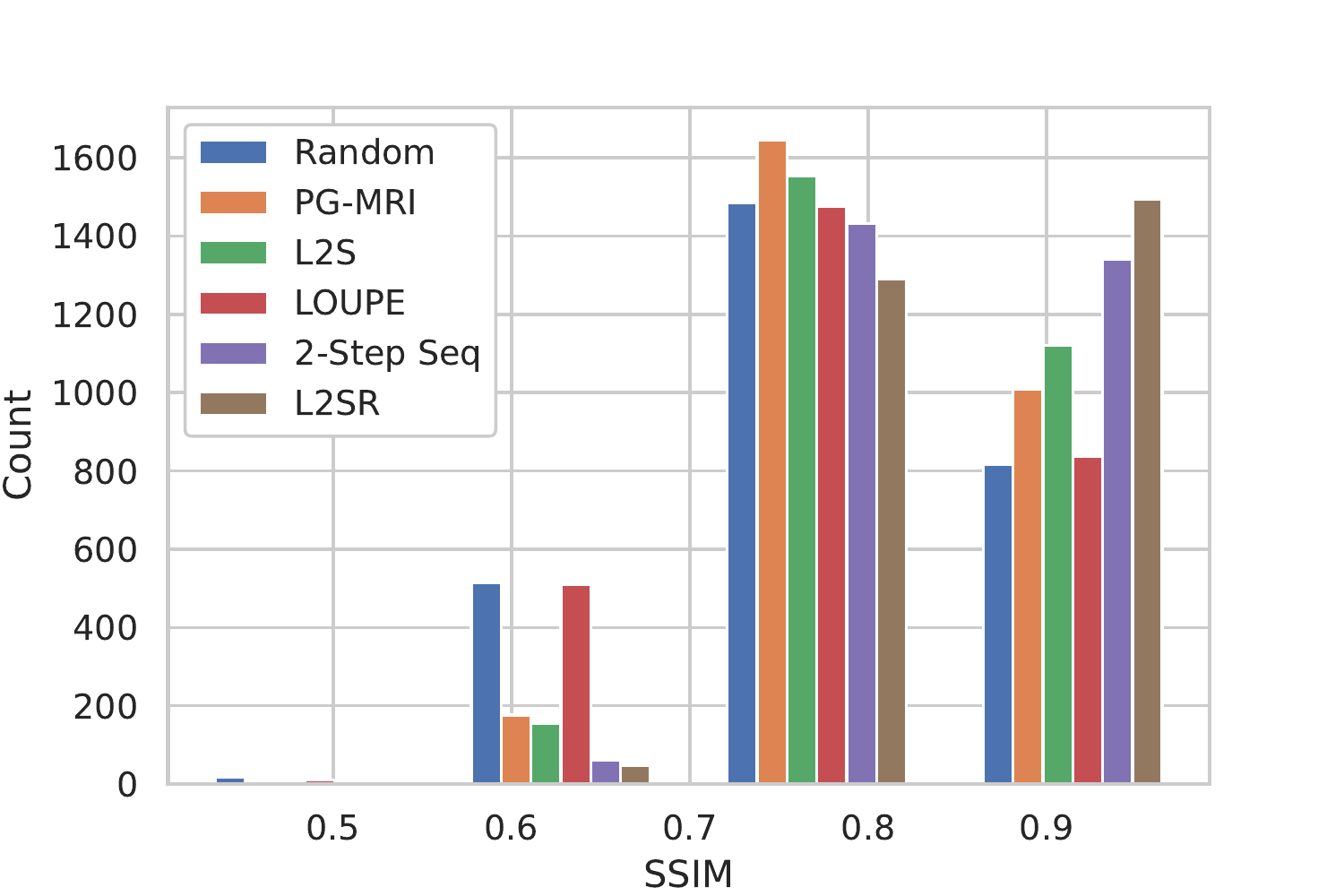}
    }
    \end{minipage}
    \caption{Histograms of SSIM values as shown in \tref{table:sparse-ssim}. Each figure contains histograms of six methods: Random, PG-MRI, L2S, LOUPE, $\tau$-Step Seq, L2SR. }
    \label{fig:histograms}
\end{figure}

\subsection{Main Results}
\label{subsec:exp-results}
\subsubsection{Reconstruction Performance}
The reconstruction results of all compared accelerated MRI methods are quantified through SSIM and PSNR values, as detailed in \tref{table:sparse-ssim}, with more granular histograms presented in \fref{fig:histograms}. The outcomes can be understood in two key aspects. First, under the 'Fixed Reconstructor' scenario, L2S consistently outperforms other methods in most cases, signifying the importance of addressing the distributional mismatch issue in enhancing reconstruction performance. Secondly, L2SR emerges as the superior method in most cases across both settings. This underlines the effectiveness of solving the overall joint optimization problem compared to merely focusing on weakened forms or suboptimization issues.

\subsubsection{Inference Complexity}
The computational complexities and corresponding inference times of trained models are presented in \tref{table:cost}. Both L2S and L2SR exhibit lower complexity and faster inference times compared to previous methods such as PG-MRI and Greedy Oracle. This empirically substantiates our claim that the proposed sparse-reward POMDP enhances efficiency by eliminating the need for intermediate reconstructions.

While the utilization of sparse rewards may raise concerns regarding increased training difficulty, the gains in inference efficiency make it a valuable tradeoff in practical applications, especially in the field of medical imaging. The focus of our experiments has been on assessing inference times, reflecting the emphasis on test efficiency in real-world deployment scenarios.

\begin{table*}[]
\footnotesize
\centering
\caption{Comparison of inference costs of different methods. We calculate the computational complexity of sampling and reconstruction required to an MRI scanning.  $C_\pi$ and $C_\mathcal{R}$ denote the computational complexity of taking one sample and one reconstruction individually. We also test the average inference time over knee dataset.}
\label{table:cost}
\begin{tabular}{llllll}
\br
\multirow{3}{*}{}               & \multirow{3}{*}{{\begin{tabular}[c]{@{}l@{}}\textbf{Computational}\\ \textbf{Complexity}\end{tabular}}} & \multicolumn{4}{c}{\textbf{Average Inference Time (s)}}                                                                                                                                        \\
                                &                                           & \multicolumn{2}{c}{$\times 4$-accelration}  & \multicolumn{2}{c}{$\times 8$-accelration}  \\
                                &                                           & Base                 & Long                 & Base                 & Long   \\ \hline
\textbf{Random} \& \textbf{LOUPE}                 & $C_{\mathcal{R}}$ & \multicolumn{4}{c}{0.0040} \\
\textbf{PG-MRI}                          & $TC_{\pi}+(T+1)C_{\mathcal{R}}$       & 0.0843 & 0.1423 & 0.0444 & 0.0727 \\
\textbf{Greedy Oracle}                   & $T(N-\frac{T-1}{2})C_{\mathcal{R}}$   & 1.2134 & 2.0440 & 0.7071 & 1.1824 \\
\textbf{L2S} \& \textbf{L2SR} (Ours)              & $TC_{\pi}+C_{\mathcal{R}}$            & 0.0201 & 0.0325 & 0.0126 & 0.0190 \\
\textbf{$\tau$-Step Seq} & $\tau C_{\pi}+(\tau+1)C_{\mathcal{R}}$           & \multicolumn{4}{c}{0.0225$(\tau=2)$, 0.0321$(\tau=4)$} \\
\br
\end{tabular}
\end{table*}

\subsection{Ablations and Further Discussions}
\subsubsection{The Best Initial Acceleration Factor}
In medical imaging, the primary objective is often to attain the finest possible reconstruction for a pre-established acceleration factor. With this focus, the initial acceleration factor is treated as a tunable hyperparameter, with its optimal value determined through empirical analysis. \Tref{table:initial-acc-factor} presents the SSIM values obtained with different initial acceleration factors. It shows that the proposed L2SR achieves the best reconstruction performance when the initial acceleration factor is established at $N/2$ for the knee dataset and $N/4$ for the brain dataset. It's worth noting that a larger initial acceleration factor grants the sampler additional degrees of freedom, leading to a superior optimal value of \eref{equ:sparse-optim-constrain}. However, it also poses a challenge by making the training more complex due to the reduced initial information available.

\begin{table}[]
\footnotesize
\centering
\caption{Empirically searching the best acceleration factors in terms of SSIM values for L2SR. We consider different settings: $\times 4$ acceleration factor for knee dataset, $\times 8$ acceleration factor for knee dataset, $\times 8$ acceleration factor for brain dataset and $\times 16$ acceleration factor for brain dataset. The best results for each settings are shown in bold numbers. }
\label{table:initial-acc-factor}
\begin{tabular}{lllll}
\br
\multirow{3}{*}{\begin{tabular}[l]{@{}l@{}}\textbf{initial}\\ \textbf{acceleration}\\ \textbf{factor}\end{tabular}} & \multicolumn{2}{c}{\multirow{2}{*}{knee}} & \multicolumn{2}{c}{\multirow{2}{*}{brain}} \\
& \multirow{2}{*}{$\times 4$-acceleration} & \multirow{2}{*}{$\times 8$-acceleration} & \multirow{2}{*}{$\times 8$-acceleration} & \multirow{2}{*}{$\times 16$-acceleration} \\
&  &  &  &  \\ \hline
$N/1$ & $0.8055 \pm 0.0327$     & $0.6418 \pm 0.0456$           & $0.8862 \pm 0.0469$           & $0.8385 \pm 0.069$            \\
$N/2$ & $\bm{0.8171 \pm 0.0316}$     & $\bm{0.6712 \pm 0.0511}$           & $0.8911 \pm 0.0452$           & $0.8456 \pm 0.0652$      \\
$N/4$ & $0.8097 \pm 0.0333$     & $0.658 \pm 0.0529$           & $\bm{0.8969 \pm 0.0417}$     & $\bm{0.8468 \pm 0.0655}$            \\
$N/8$ & $0.7999 \pm 0.036$     & $0.6332 \pm 0.053$           & $0.893 \pm 0.043$           & $0.8404 \pm 0.066$      \\
$N/16$ & $0.7681 \pm 0.0416$     & ---          & $0.8899 \pm 0.043$     & ---           \\
\br
\end{tabular}
\end{table}

\begin{table}
    \caption{Influence of discount factor to the proposed sparse-reward POMDP. We show SSIM values of L2S with respect to discount factors under $\times 4$-acceleration on the knee test dataset. }
    \label{table:exp-gamma}
    \begin{indented}
    \lineup
    \item[]\begin{tabular}{lll}
    \br
     & \textbf{Base-horizon} & \textbf{Long-horizon} \\ \hline
        $\gamma = 0.5$ & $0.7129 \pm 0.0372$ & $0.7065 \pm 0.0298$ \\
        $\gamma = 0.9$ & $0.7528 \pm 0.0363$ & $0.7826 \pm 0.0264$ \\ 
        $\gamma = 1.0$ & $\bm{0.7543 \pm 0.0372}$ & $\bm{0.7838 \pm 0.0286}$ \\ 
    \br
    \end{tabular}
    \end{indented}
\end{table}

\subsubsection{Discount Factor}
In \tref{table:exp-gamma}, we present the impact of the discount factor on L2S. Our proposed L2S method achieves the best reconstruction performance when $\gamma=1$. In contrast, previous dynamic sampling methods with dense-reward POMDP, as described in \cite{pineda2020active} and \cite{NEURIPS2020_daed2103}, have shown their optimal performance at discount factors of 0.5 and 0.9, respectively. The contrastive results make sense as the proposed sparse-reward POMDP is explicitly designed to get the reward at the end of the trajectory, thereby learning a long-sighted policy.

\subsubsection{Number of Alternation}
The L2SR uses a fixed number of alternations as its stopping condition. The reconstruction performance over rounds of alternating training is depicted in \fref{fig:num-alter}. Typically, the algorithm converges within $L=5$.

\subsubsection{Adaptability Across Sampling Schemes}
A significant benefit of our framework lies in its extensibility across different sampling schemes. While the core of our analysis and experiments employs 1D vertical Cartesian sampling, the proposed framework could be applied to a wider range of applications in more realistic pulse sequence design. To demonstrate its adaptability, we have extended our method to radial sampling \cite{lustig2007sparse}, where the selection changes from choosing a column to picking an angle in k-space. The core training process remains unchanged. The result in \fref{fig:exten-sample} confirms our method's superiority over baselines within radial sampling schemes, which highlights the extensibility of our framework across various acquisition patterns.

\begin{figure}[t]
    \centering
    \subfigure[$\times 4$-acceleration.]{
        \includegraphics[width=0.40\columnwidth]{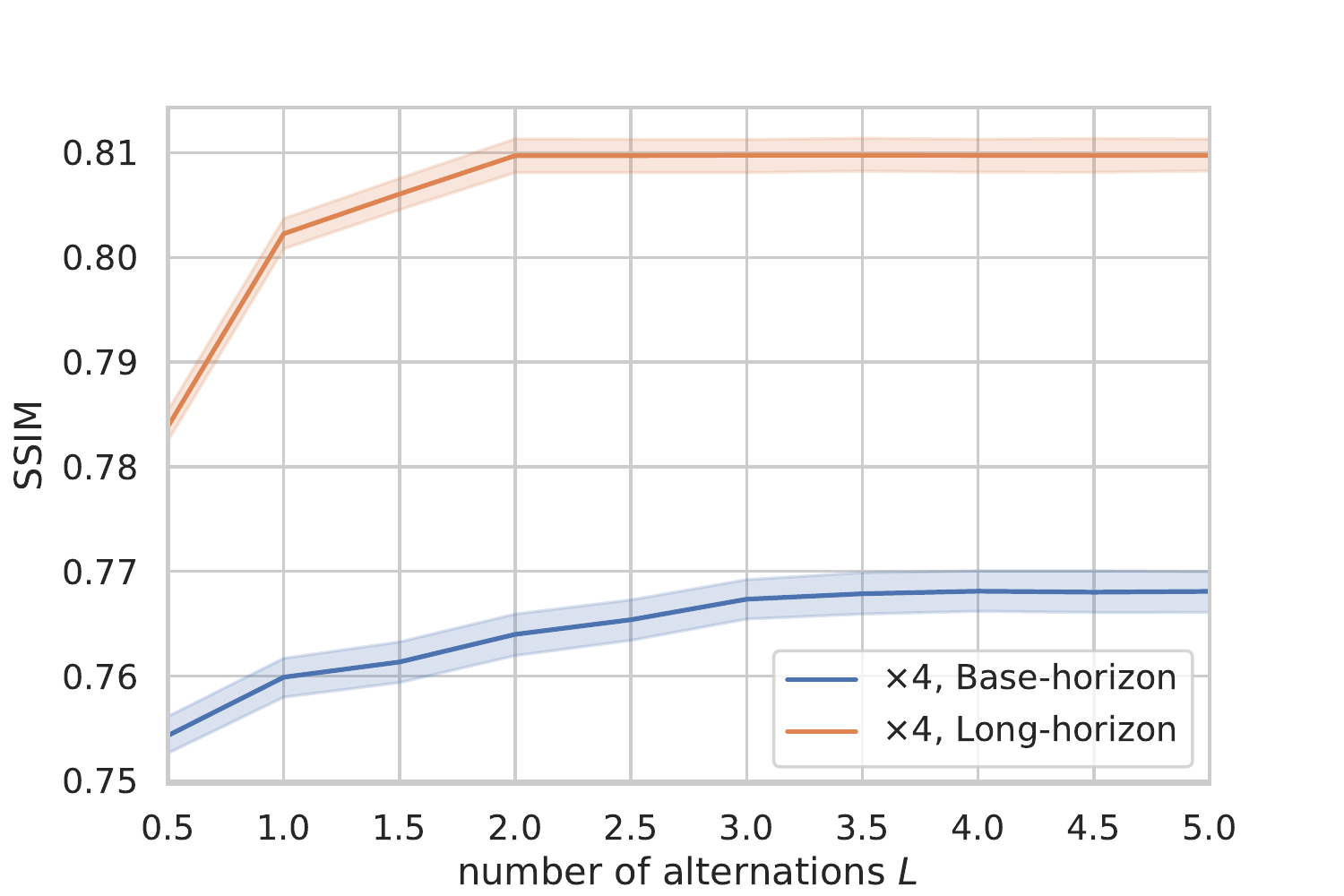}
    }
    \subfigure[$\times 8$-acceleration.]{
        \includegraphics[width=0.40\columnwidth]{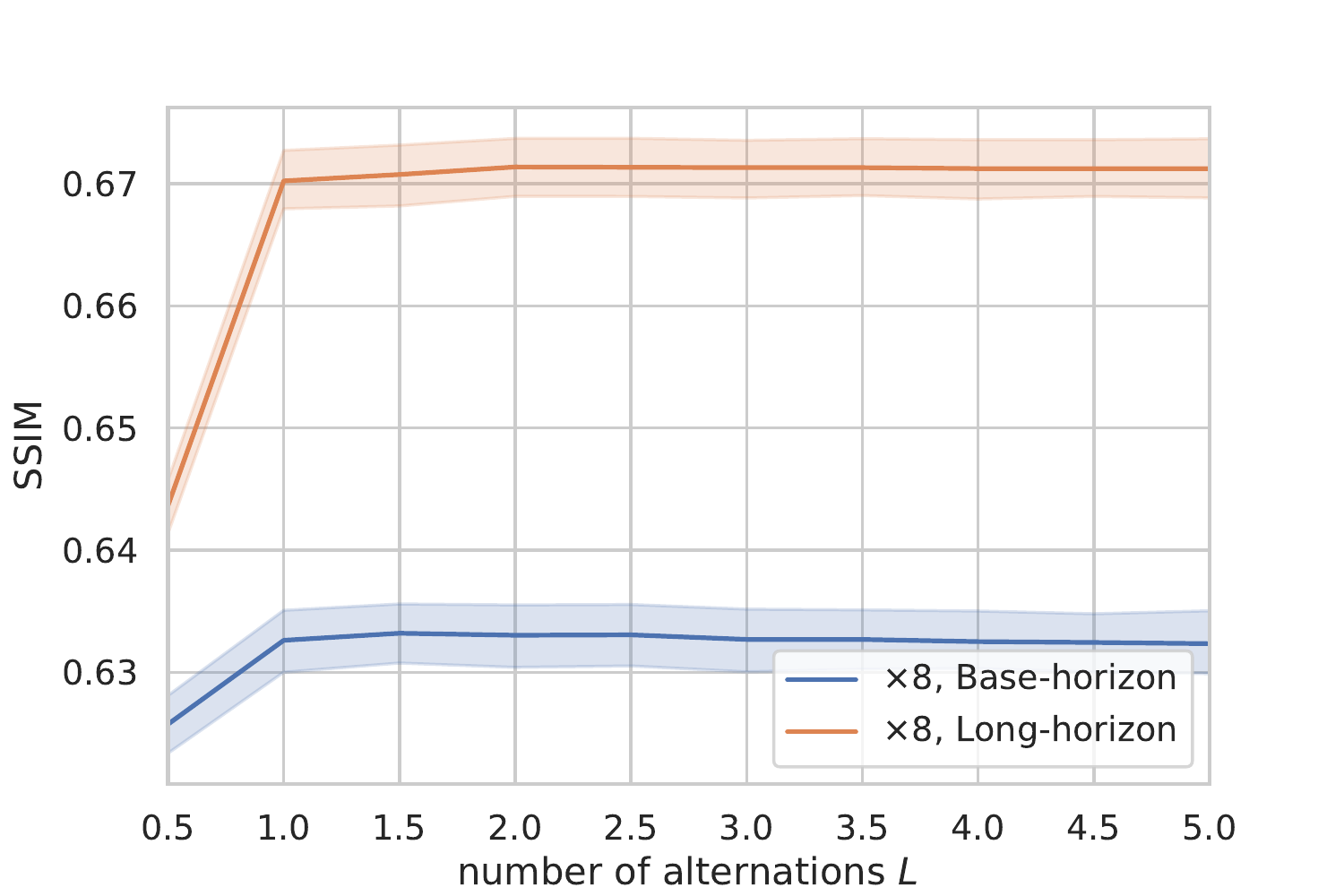}
    }
    \caption{The choice of the number of alternation. We show SSIM values of L2SR with respect to different rounds of alternating training on the knee test dataset. Specifically,
    "$L=l.5$" means the $(l+1)$th round of training the sampler with the learned reconstructor, i.e. solving \eref{equ:alter-suboptim2}. }
    \label{fig:num-alter}
\end{figure}

\begin{figure}
    \centering
    \subfigure[$\times 4$-acceleration.]{
        \includegraphics[width=0.40\columnwidth]{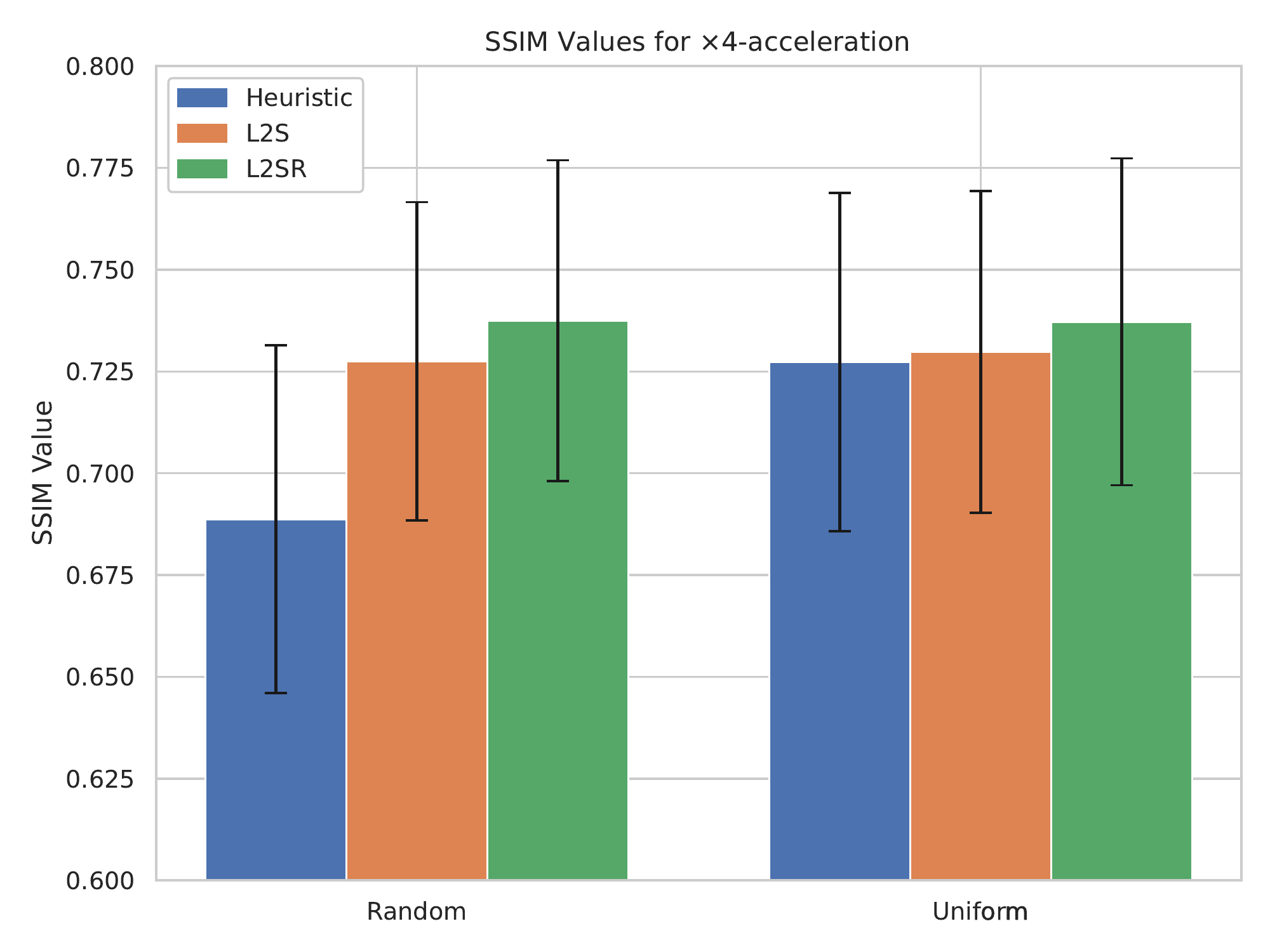}
    }
    \subtable[$\times 8$-acceleration.]{
        \includegraphics[width=0.40\columnwidth]{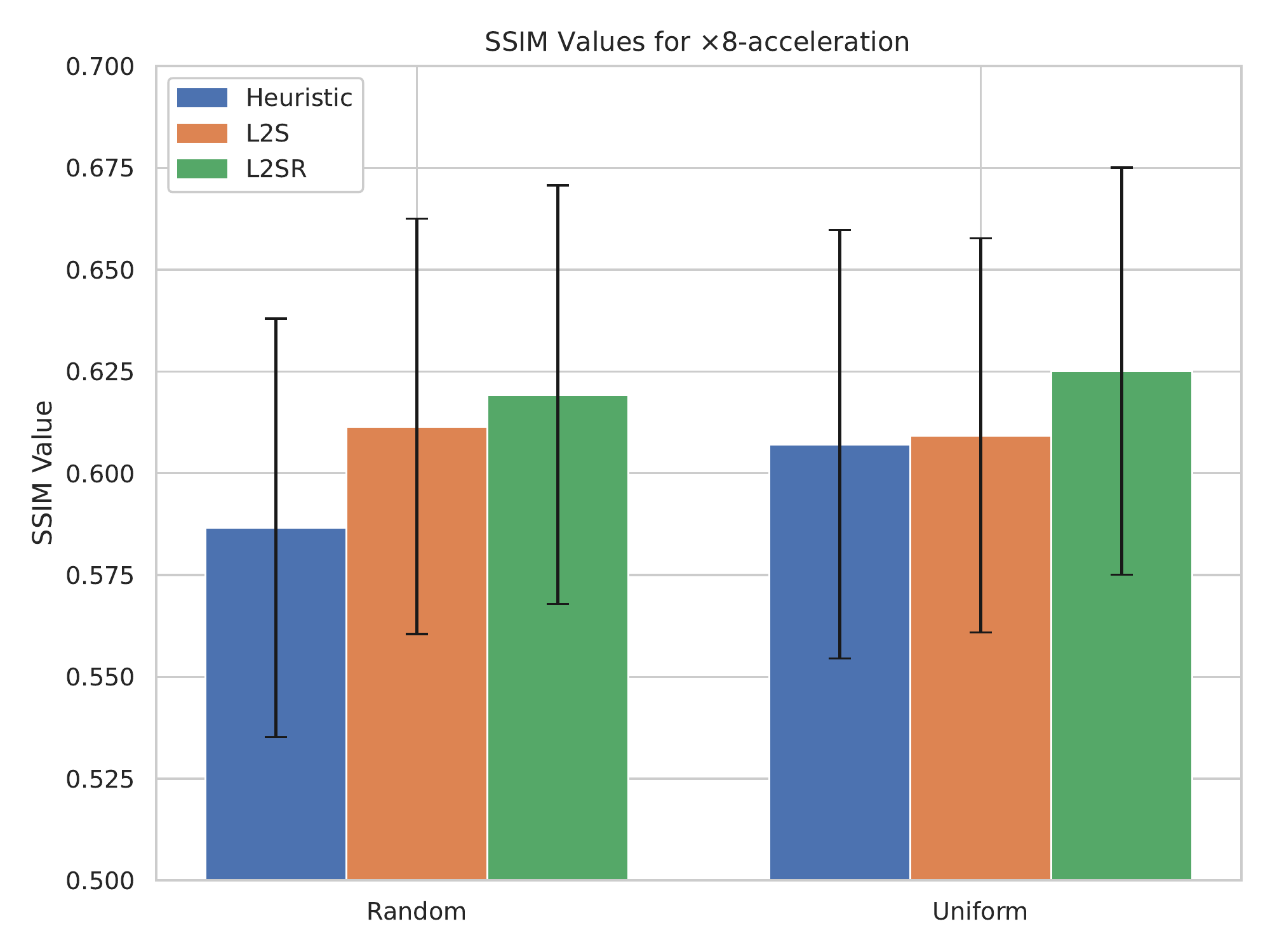}
    }
    \caption{Comparison of SSIM values using the radial sampling pattern on the knee test dataset. The bars show the mean SSIM values, while the error bars indicate variances. We evaluate our L2S and L2SR methods against a pre-trained reconstructor baseline, comparing under two distinct initial heuristic sampling policies: Random and Uniform.}
    \label{fig:exten-sample}
\end{figure}


\subsection{Visualization}
In \fref{fig:visualization_knee}, we present a sampling and reconstruction example of the knee dataset under  $\times 4$-acceleration, while in \fref{fig:visualization_brain}, we demonstrate the same for the brain dataset under $\times 8$-acceleration. A comparison of these examples reveals that the proposed L2SR method produces images of superior visual quality. Notably, the images generated through L2SR exhibit fewer artifacts and bear a closer visual resemblance to the ground truth compared to those created by competing methods. 

\begin{figure}
    \centering
    \subfigure[Random]{
        \begin{minipage}[b]{0.125\linewidth}
        \includegraphics[width=\linewidth]{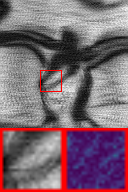}\vspace{4pt}
        \includegraphics[width=\linewidth]{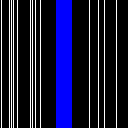}\vspace{4pt}
        \includegraphics[width=\linewidth]{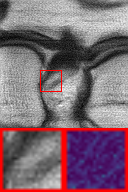}\vspace{4pt}
        \includegraphics[width=\linewidth]{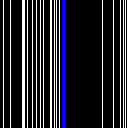}
        \end{minipage}
    }\hspace{-4pt}
    \subfigure[PG-MRI]{
        \begin{minipage}[b]{0.125\linewidth}
        \includegraphics[width=\linewidth]{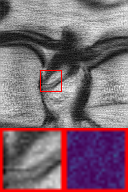}\vspace{4pt}
        \includegraphics[width=\linewidth]{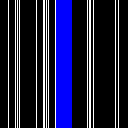}\vspace{4pt}
        \includegraphics[width=\linewidth]{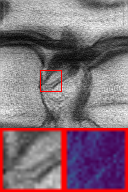}\vspace{4pt}
        \includegraphics[width=\linewidth]{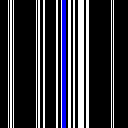}
        \end{minipage}
    }\hspace{-4pt}
    \subfigure[L2S]{
        \begin{minipage}[b]{0.125\linewidth}
        \includegraphics[width=\linewidth]{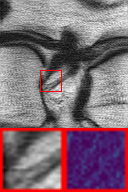}\vspace{4pt}
        \includegraphics[width=\linewidth]{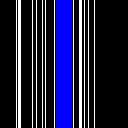}\vspace{4pt}
        \includegraphics[width=\linewidth]{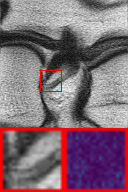}\vspace{4pt}
        \includegraphics[width=\linewidth]{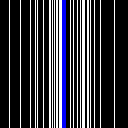}
        \end{minipage}
    }\hspace{-4pt}
    \subfigure[LOUPE]{
        \begin{minipage}[b]{0.125\linewidth}
        \includegraphics[width=\linewidth]{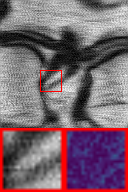}\vspace{4pt}
        \includegraphics[width=\linewidth]{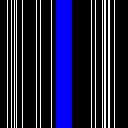}\vspace{4pt}
        \includegraphics[width=\linewidth]{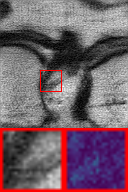}\vspace{4pt}
        \includegraphics[width=\linewidth]{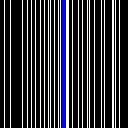}
        \end{minipage}
    }\hspace{-4pt}
    \subfigure[$\tau$-Step Seq]{
        \begin{minipage}[b]{0.125\linewidth}
        \includegraphics[width=\linewidth]{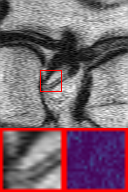}\vspace{4pt}
        \includegraphics[width=\linewidth]{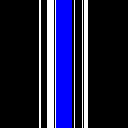}\vspace{4pt}
        \includegraphics[width=\linewidth]{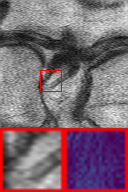}\vspace{4pt}
        \includegraphics[width=\linewidth]{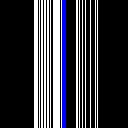}
        \end{minipage}
    }\hspace{-4pt}
    \subfigure[L2SR]{
        \begin{minipage}[b]{0.125\linewidth}
        \includegraphics[width=\linewidth]{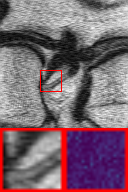}\vspace{4pt}
        \includegraphics[width=\linewidth]{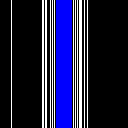}\vspace{4pt}
        \includegraphics[width=\linewidth]{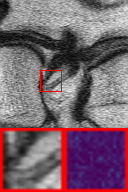}\vspace{4pt}
        \includegraphics[width=\linewidth]{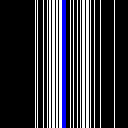}
        \end{minipage}
    }\hspace{-4pt}
    \subfigure[GT]{
        \begin{minipage}[b]{0.125\linewidth}
        \includegraphics[width=\linewidth]{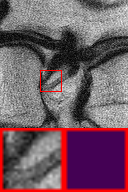}\vspace{4pt}
        \includegraphics[width=\linewidth]{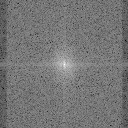}\vspace{4pt}
        \includegraphics[width=\linewidth]{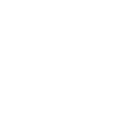}\vspace{-1pt}
        \includegraphics[width=0.5\linewidth]{Experiments/visualization/figure7g4.png}\vspace{4pt}
        \includegraphics[width=\linewidth]{Experiments/visualization/figure7g4.png} 
        
        \end{minipage}
    }
    
    \caption{Visualisation of an example of the knee dataset for our methods and other competing methods under $\times 4$-acceleration. From left to right, the columns display various accelerated MRI methods: Random, PG-MRI, L2S, LOUPE, $\tau$-Step Seq, and L2SR, with the final column showing the ground truth (GT) and fully sampled k-space. From top to bottom, The rows represent: Base-horizon reconstruction images, Base-horizon masks, Long-horizon reconstruction images, and Long-horizon masks. An enlarged specific region is presented at the bottom left of each reconstruction image, along with its corresponding error map at the bottom right. The blue marks indicate measurements sampled from the low-frequency part of k-space in the first state. }
    \label{fig:visualization_knee}
\end{figure}
\begin{figure}
    \centering
    \subfigure[Random]{
        \begin{minipage}[b]{0.125\linewidth}
        \includegraphics[width=\linewidth]{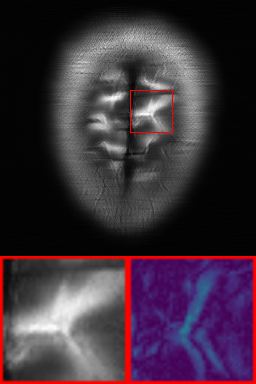}\vspace{4pt}
        \includegraphics[width=\linewidth]{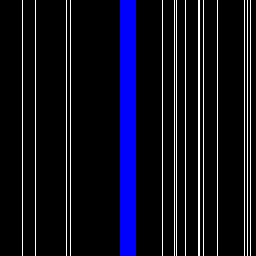}\vspace{4pt}
        \includegraphics[width=\linewidth]{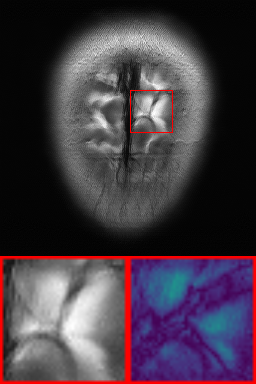}\vspace{4pt}
        \includegraphics[width=\linewidth]{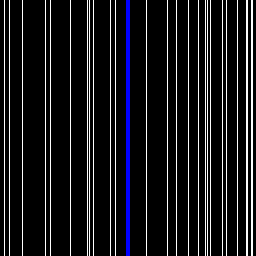}
        \end{minipage}
    }\hspace{-4pt}
    \subfigure[PG-MRI]{
        \begin{minipage}[b]{0.125\linewidth}
        \includegraphics[width=\linewidth]{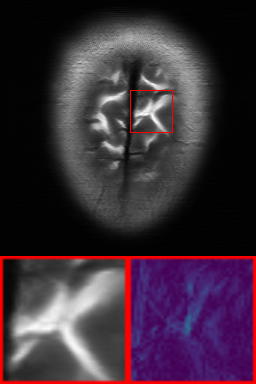}\vspace{4pt}
        \includegraphics[width=\linewidth]{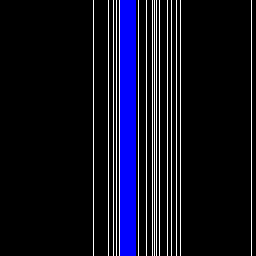}\vspace{4pt}
        \includegraphics[width=\linewidth]{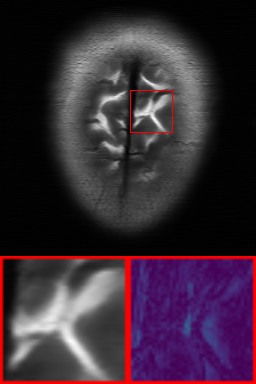}\vspace{4pt}
        \includegraphics[width=\linewidth]{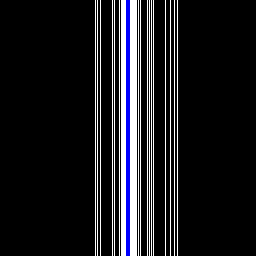}
        \end{minipage}
    }\hspace{-4pt}
    \subfigure[L2S]{
        \begin{minipage}[b]{0.125\linewidth}
        \includegraphics[width=\linewidth]{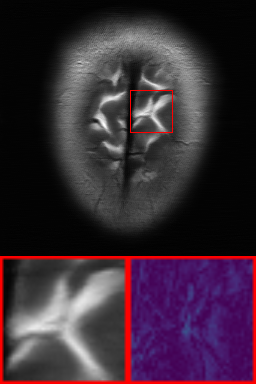}\vspace{4pt}
        \includegraphics[width=\linewidth]{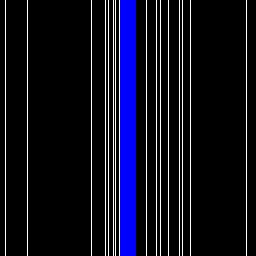}\vspace{4pt}
        \includegraphics[width=\linewidth]{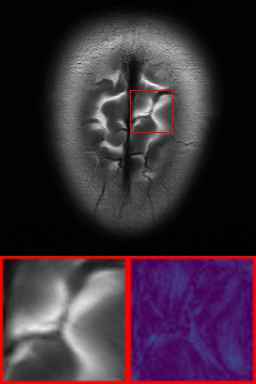}\vspace{4pt}
        \includegraphics[width=\linewidth]{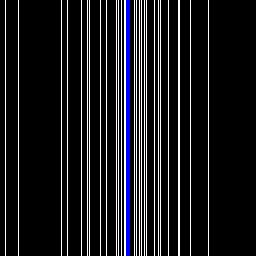}
        \end{minipage}
    }\hspace{-4pt}
    \subfigure[LOUPE]{
        \begin{minipage}[b]{0.125\linewidth}
        \includegraphics[width=\linewidth]{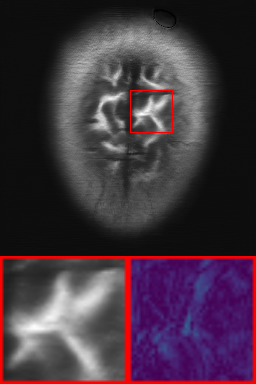}\vspace{4pt}
        \includegraphics[width=\linewidth]{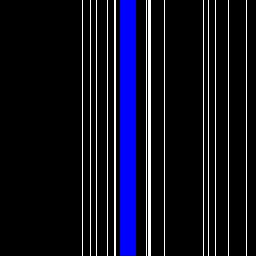}\vspace{4pt}
        \includegraphics[width=\linewidth]{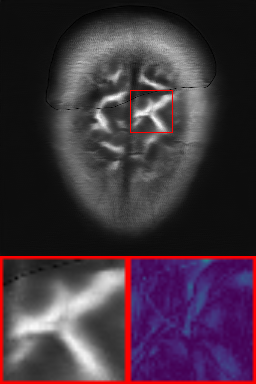}\vspace{4pt}
        \includegraphics[width=\linewidth]{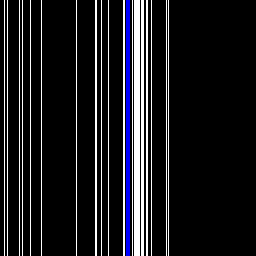}
        \end{minipage}
    }\hspace{-4pt}
    \subfigure[$\tau$-Step Seq]{
        \begin{minipage}[b]{0.125\linewidth}
        \includegraphics[width=\linewidth]{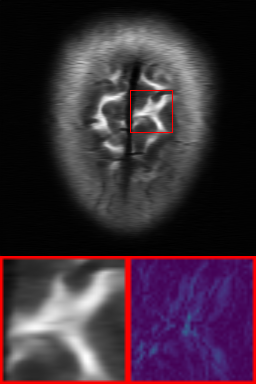}\vspace{4pt}
        \includegraphics[width=\linewidth]{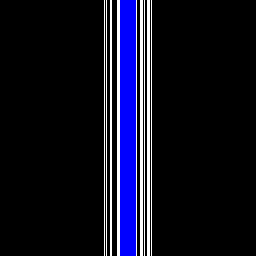}\vspace{4pt}
        \includegraphics[width=\linewidth]{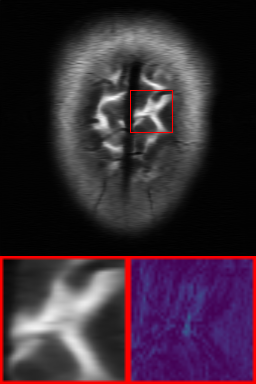}\vspace{4pt}
        \includegraphics[width=\linewidth]{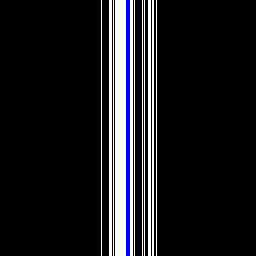}
        \end{minipage}
    }\hspace{-4pt}
    \subfigure[L2SR]{
        \begin{minipage}[b]{0.125\linewidth}
        \includegraphics[width=\linewidth]{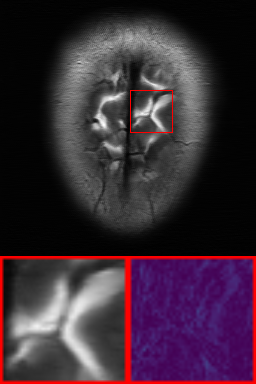}\vspace{4pt}
        \includegraphics[width=\linewidth]{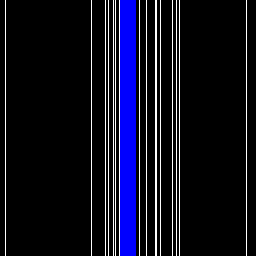}\vspace{4pt}
        \includegraphics[width=\linewidth]{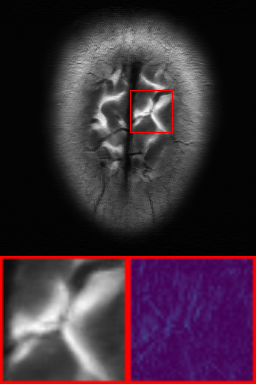}\vspace{4pt}
        \includegraphics[width=\linewidth]{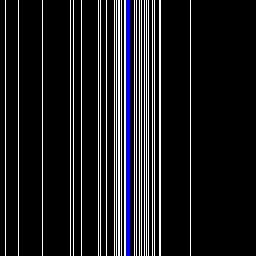}
        \end{minipage}
    }\hspace{-4pt}
    \subfigure[GT]{
        \begin{minipage}[b]{0.125\linewidth}
        \includegraphics[width=\linewidth]{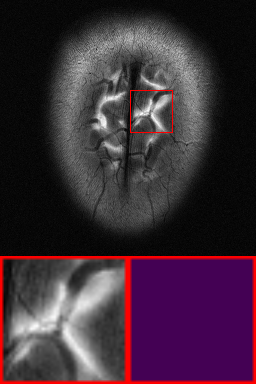}\vspace{4pt}
        \includegraphics[width=\linewidth]{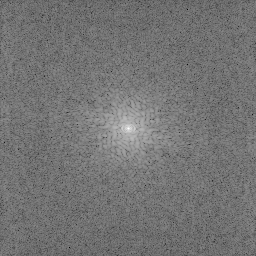}\vspace{4pt}
        \includegraphics[width=\linewidth]{Experiments/visualization/figure7g4.png}\vspace{-1pt}
        \includegraphics[width=0.5\linewidth]{Experiments/visualization/figure7g4.png}\vspace{4pt}
        \includegraphics[width=\linewidth]{Experiments/visualization/figure7g4.png} 
        
        \end{minipage}
    }
    \caption{Visualisation of an example of the brain dataset for our methods and other competing methods under $\times 8$-acceleration.}
    \label{fig:visualization_brain}
\end{figure}

\section{Conclusion and Discussion}
\label{sec:concl}
In this paper, we proposed a novel alternating training framework to jointly learn a personalized sampling policy and a corresponding reconstruction model for accelerated MRI. Specifically, we formulated the MRI sampling trajectory as a sparse-reward POMDP to learn a personalized sampler. Compared to existing dynamic sampling methods that utilize dense-reward POMDP, our proposed sparse-reward POMDP is more computationally efficient and avoids the distributional mismatch. Furthermore, the proposed framework, called L2SR, solve the joint optimization problem of learning a pair of samplers and reconstructors, thus eliminating the training mismatch. 

Our empirical results on the fastMRI dataset demonstrate that the sparse-reward POMDP improves reconstruction quality by eliminating the distributional mismatch. Moreover, L2SR achieves the best acceleration-quality trade-off and inferences much faster than existing dynamic sampling methods. Overall, our proposed method provides a promising solution for accelerating MRI with high reconstruction quality.

However, our work has two limitations. Firstly, the alternating training framework has no theoretical guarantee of convergence. Secondly, within the actual exploration mechanism, there is a lack of diversity in the policy's actions towards the sample to solve \eref{equ:alter-suboptim2}, which is a fundamental challenge associated with the sparse-reward POMDP. Through empirical observation, we have identified that as the length of the sample increases, the complexity of training a sparse-reward POMDP escalates significantly. In sparse-reward POMDP, the paucity of frequent informative feedback limits the effective exploration from the environment. However, current reinforcement learning algorithms struggle to fully mitigate these issues. In future work, we plan to address these limitations and adapt our approach to realistic MRI scanning problems.

\section*{Data availability statement}
The data that support the findings of this study are openly available  at the following URL: \url{https://fastmri.med.nyu.edu/}. 

\section*{Acknowledgment}
This work was supported by the National Natural Science Foundation of China under Grant 12090022.

\newpage
\appendix
\section{Notation Index}
\label{appen-sec:notation}
We give a list of notation index and where it is defined in \tref{table:notation}. 
\begin{table}[!h]
    \caption{Notation index.}
    \label{table:notation}
    \begin{indented}
    \lineup
    \item[]
    \begin{tabular}{l|c|l}
    \br
        \textbf{Noun} & \textbf{Notation} & \textbf{Reference} \\ \hline
        image width                       & $N$                  & \sref{subsec:notation} \\ 
        ground truth                      & $\mathbf{x}$         & \sref{subsec:notation} \\ 
        2D Discrete Fourier Transform     & $\mathcal{F}$        & \sref{subsec:notation} \\ 
        fully sampled measurement         & $\mathbf{y}$         & \sref{subsec:notation} \\ 
        binary mask matrix                & $\mathbf{M}$         & \sref{subsec:notation} \\ 
        reconstructor                     & $\mathcal{R}$        & \sref{subsec:notation} \\ 
        heuristic sampling policy         & $\pi_t^{\text{h}}$              & \sref{subsec:notation} \\ 
        binary mask vector                & $\mathbf{a}$         & \sref{subsec:notation} \\
        dense-reward joint optimization problem &  & \eref{equ:dense-optim-constrain} \\
        dataset & $\mathcal{D}$ & \eref{equ:dense-optim-constrain}  \\
        sampler                           & $\pi$                & \eref{equ:dense-optim-constrain} \\ 
        N-dimensional probability simplex &  $\Delta^N$  & \eref{equ:dense-optim-constrain}  \\
        reconstruction suboptimization problem &  & \eref{equ:heuristic-suboptim} \\
        dense-reward suboptimization problem &  & \eref{equ:dense-suboptim} \\
        dense-reward POMDP                &                      & \sref{subsec:optim} \\
        observation                       & $\mathbf{y}_t$ & \sref{subsec:optim} \\ 
        discount factor                   & $\gamma$             & \sref{subsec:optim} \\
        sparse-reward POMDP &  & \sref{subsec:sparse-reward} \\
        sparse-reward joint optimization problem &  & \eref{equ:sparse-optim-constrain} \\
        continuous function space &  $C$  &  theorem \ref{theorem:non-deterministic} and \ref{theorem:improved-dynamic}  \\
        sparse-reward suboptimization problem &  & \eref{equ:sparse-suboptim} \\
        Fixed Reconstructor &  & \sref{subsec:exp-implementation} \\
        Joint Training &  & \sref{subsec:exp-implementation} \\
        acceleration factor               &  & \sref{subsec:exp-implementation} \\
        initial acceleration factor       &  & \sref{subsec:exp-implementation} \\
        Base-horizon &  & \sref{subsec:exp-implementation} \\
        Long-horizon &  & \sref{subsec:exp-implementation} \\
    \br
    \end{tabular}
    \end{indented}
\end{table}

\section{Supplements and Proofs for the Theoretic Analysis}
\label{appen-sec:theory}

\subsection{Proof of theorem \ref{theorem:non-deterministic}}
\label{appen-subsec:proof-non-deterministic}

We prove it with the idea of dynamic programming. 

\begin{proof}
Let $v_{\pi, \mathcal{R}}^{\text{dense}}(\mathbf{x}, \mathbf{M}_t)$ and $v_{\pi, \mathcal{R}}^{\text{sparse}}(\mathbf{x}, \mathbf{M}_t)$ be the value function of dense-reward POMDP and sparse-reward POMDP, defined as
$$
v_{\pi, \mathcal{R}}^{\text{dense}}(\mathbf{x}, \mathbf{M}_t) = \mathbb{E}_{\pi} \sum_{s=t+1}^T r_s, \quad \text{where} \ \mathbf{y}_t = \mathbf{M}_t \odot \mathcal{F}(\mathbf{x}),
$$$$
v_{\pi, \mathcal{R}}^{\text{sparse}}(\mathbf{x}, \mathbf{M}_t) = \mathbb{E}_{\pi} r_T, \quad \text{where} \ \mathbf{y}_t = \mathbf{M}_t \odot \mathcal{F}(\mathbf{x}).
$$

We define $v_{\pi, \mathcal{R}}^{\text{dense}}(\mathcal{D}, \mathbf{M}_t) = \mathbb{E}_{\mathbf{x}\sim\mathcal{D}} v_{\pi, \mathcal{R}}^{\text{dense}}(\mathbf{x}, \mathbf{M}_t)$ and $v_{\pi, \mathcal{R}}^{\text{sparse}}(\mathcal{D}, \mathbf{M}_t) = \mathbb{E}_{\mathbf{x}\sim\mathcal{D}} v_{\pi, \mathcal{R}}^{\text{sparse}}(\mathbf{x}, \mathbf{M}_t)$. Indeed, $v_{\pi, \mathcal{R}}^{\text{dense}}(\mathcal{D}, \mathbf{M}_0)$ is the objective function of \eref{equ:dense-optim-constrain} and $v_{\pi, \mathcal{R}}^{\text{sparse}}(\mathcal{D}, \mathbf{M}_0)$ is the objective function of \eref{equ:sparse-optim-constrain}. We want to prove:
$$
\sup_{\substack{\pi \in C(\mathbb{R}^{N\times N}, \Delta^N) \\ \mathcal{R} \in C(\mathbb{C}^{N\times N}, \mathbb{R}^{N\times N})}} v_{\pi, \mathcal{R}}^{\text{dense}}(\mathcal{D}, \mathbf{M}_0) \leq \sup_{\substack{\pi \in C(\mathbb{C}^{N\times N}, \Delta^N) \\ \mathcal{R} \in C(\mathbb{C}^{N\times N}, \mathbb{R}^{N\times N})}} v_{\pi, \mathcal{R}}^{\text{sparse}}(\mathcal{D}, \mathbf{M}_0).
$$

For $\forall \epsilon > 0$, there exists a pair of $\pi_\epsilon \in C(\mathbb{R}^{N\times N}, \Delta^N)$ and $\mathcal{R}_\epsilon \in C(\mathbb{C}^{N\times N}, \mathbb{R}^{N\times N})$ satisfying
$$
\sup_{\substack{\pi \in C(\mathbb{R}^{N\times N}, \Delta^N) \\ \mathcal{R} \in C(\mathbb{C}^{N\times N}, \mathbb{R}^{N\times N})}} v_{\pi, \mathcal{R}}^{\text{dense}}(\mathcal{D}, \mathbf{M}_0) - v_{\pi_\epsilon, \mathcal{R}_\epsilon}^{\text{dense}}(\mathcal{D}, \mathbf{M}_0) < \epsilon.
$$

We have the recursion formula for both $v_{\pi, \mathcal{R}}^{\text{dense}}(\mathbf{x}, \mathbf{M}_0)$ and $v_{\pi, \mathcal{R}}^{\text{sparse}}(\mathbf{x}, \mathbf{M}_0)$:
$$
v_{\pi, \mathcal{R}}^{\text{dense}}(\mathbf{x}, M_t) = \sum_{a_t=1}^{N} \pi(a_t \mid \mathbf{x}_t) \cdot v_{\pi, \mathcal{R}}^{\text{dense}}(\mathbf{x}, \mathbbm{1}(\mathbf{M}_t+\mathbf{M}^{a_t})), 
$$$$
v_{\pi, \mathcal{R}}^{\text{sparse}}(\mathbf{x}, M_t) = \sum_{a_t=1}^{N} \pi(a_t \mid \mathbf{y}_t) \cdot v_{\pi, \mathcal{R}}^{\text{sparse}}(\mathbf{x}, \mathbbm{1}(\mathbf{M}_t+\mathbf{M}^{a_t})). 
$$

Let $\pi_\epsilon^* \in C(\mathbb{C}^{N\times N}, \Delta^N)$ defined as
$$
\pi_\epsilon^*(\cdot \mid \mathbf{y}) = \pi_\epsilon(\cdot \mid \mathcal{R}_\epsilon(\mathbf{y})), \quad \forall \mathbf{y} \in \mathbb{C}^{N\times N}.
$$

Since the similarity metric function $\mathrm{S}$ is bounded by $[0, 1]$, $v_{\pi, \mathcal{R}} \in [0, 1]$. Recursively, we have 
\begin{align*}
    v^{\text{dense}}_{\pi_\epsilon, \mathcal{R}_\epsilon}(\mathcal{D}, \mathbf{M}_T) & = \mathbb{E}_{\mathbf{x}\sim\mathcal{D}} \mathrm{S}(\mathcal{R}_\epsilon(\mathbf{M}_T \odot \mathcal{F}(\mathbf{x})), \mathbf{x}) \\
    & = v^{\text{sparse}}_{\pi_\epsilon^*, \mathcal{R}_\epsilon}(\mathcal{D}, \mathbf{M}_T),
\end{align*}
\begin{align*}
    & v^{\text{dense}}_{\pi_\epsilon, \mathcal{R}\epsilon}(\mathcal{D}, \mathbf{M}_{T-1}) \\ 
    = \ & \mathbb{E}_{\mathbf{x}\sim\mathcal{D}} \sum_{a_{T-1}=1}^{n}  \pi_\epsilon(a_t \mid \mathbf{x}_{T-1}) \cdot v^{\text{dense}}_{\pi_\epsilon, \mathcal{R}_\epsilon}(\mathbf{x}, \mathbbm{1}(\mathbf{M}_{T-1}+\mathbf{M}^{a_{T-1}})) \\
    = \ & \mathbb{E}_{\mathbf{x}\sim\mathcal{D}} \sum_{a_{T-1}=1}^{n}  \pi_\epsilon^*(a_t \mid \mathbf{y}_{T-1}) \cdot v^{\text{dense}}_{\pi_\epsilon, \mathcal{R}_\epsilon}(\mathbf{x}, \mathbbm{1}(\mathbf{M}_{T-1}+\mathbf{M}^{a_{T-1}})) \\
    = \ & \mathbb{E}_{x\sim\mathcal{D}} \sum_{a_{T-1}=1}^{n} \pi_\epsilon^*(a_t \mid \mathbf{y}_{T-1}) \cdot v^{\text{sparse}}_{\pi_\epsilon^*, \mathcal{R}_\epsilon}(\mathbf{x}, \mathbbm{1}(\mathbf{M}_{T-1}+\mathbf{M}^{a_{T-1}})) \\
    = \ & v^{\text{sparse}}_{\pi_\epsilon^*, \mathcal{R}_\epsilon}(\mathcal{D}, \mathbf{M}_{T-1}),
\end{align*}
$$
\cdots
$$

Finally, we have
$$
v^{\text{dense}}_{\pi_\epsilon}(\mathcal{D}, \mathbf{M}_0) = v^{\text{sparse}}_{\pi_\epsilon^*}(\mathcal{D}, \mathbf{M}_0),
$$
which means that for $\forall \epsilon$, there exists a pair of $\pi_\epsilon^* \in C(\mathbb{C}^{N\times N}, \Delta^N)$ and $\mathcal{R}_\epsilon \in C(\mathbb{C}^{N\times N}, \mathbb{R}^{N\times N})$ satisfying
$$
\sup_{\substack{\pi \in C(\mathbb{R}^{N\times N}, \Delta^N) \\ \mathcal{R} \in C(\mathbb{C}^{N\times N}, \mathbb{R}^{N\times N})}} v_{\pi, \mathcal{R}}^{\text{dense}}(\mathcal{D}, \mathbf{M}_0) < v_{\pi_\epsilon^*, \mathcal{R}_\epsilon}^{\text{sparse}}(\mathcal{D}, \mathbf{M}_0) + \epsilon.
$$

Because of the arbitrariness of $\epsilon$, the proof is completed! 
\end{proof}

\subsection{Proof of theorem \ref{theorem:improved-dynamic}}
\label{appen-subsec:dynamic}

\begin{proof}
Without loss of generality, suppose that $\mathbf{M}_0 = \mathbf{0}$. 

We denote
$$
    J(\mathbf{a}, \mathcal{R}) = \mathbb{E}_{\mathbf{x}\sim\mathcal{D}} \mathrm{S}(\mathcal{R}(\mathbf{M}^{\mathbf{a}} \odot \mathcal{F}(\mathbf{x})), \mathbf{x}).
$$

We first prove that for $\forall \Vert \mathbf{a} \Vert_1 = T$, we have
$$
    J(\mathbf{a}, \mathcal{R}^{\text{dense}}) \leq J(\mathbf{a}, \mathcal{R}^{\text{sparse}}),
$$
by contradiction. 

If it does not hold, there $\exists \ \Vert \mathbf{a}^\star \Vert_1 = T$ that 
$$
    J(\mathbf{a}^\star, \mathcal{R}^{\text{dense}}) > J(\mathbf{a}^\star, \mathcal{R}^{\text{sparse}}).
$$
Then, we prove that it contradicts to the the maximum property of $\mathcal{R}^{\text{sparse}}$. 

Notice that the input of $\mathcal{R}$ is $\widetilde{\mathbf{y}} = \mathbf{M}^{\mathbf{a}} \odot \mathbf{y}$, so their exists a well-defined function $\mathscr{A}: \mathbb{C}^{N\times N} \rightarrow \{0, 1\}^N, \widetilde{\mathbf{y}} \rightarrow \mathbf{a}$. Specifically, values of $\mathbf{a}$ at the indicators of zero columns of $\widetilde{\mathbf{y}}$ equal to 0, and others equal to 1. We define $\mathscr{Y}: \mathbb{C}^{N\times N} \rightarrow \mathbb{R}^+, \widetilde{\mathbf{y}} \rightarrow \min_{\mathscr{A}(\widetilde{\mathbf{y}})(a_t) = 1} \{ \Vert \lvert \mathbf{M}^{a_t} \odot \widetilde{\mathbf{y}} \rvert \Vert_2 \}$ as the minimal norm of the non-zero column of $\widetilde{\mathbf{y}}$. 

We construct a new reconstructor on two disjoint closed subsets of $\mathbb{C}^{N \times N}$:
\begin{equation*}
\mathcal{R}_\epsilon^{\text{new}}(\widetilde{\mathbf{y}}) = \left\{
\begin{aligned}
& \mathcal{R}^{\text{sparse}}, \quad \Vert \mathscr{A}(\widetilde{\mathbf{y}}) \Vert_1 \leq T \ \text{ and } \ \mathscr{A}(\widetilde{\mathbf{y}}) \neq \mathbf{a}^\star \\
& \mathcal{R}^{\text{dense}}, \quad \mathscr{A}(\widetilde{\mathbf{y}}) = \mathbf{a}^\star \ \text{ and } \ \mathscr{Y}(\widetilde{\mathbf{y}}) \geq \epsilon
\end{aligned},
\right.
\end{equation*}
where $\epsilon > 0$ is pending. (We give a supplementary description for the convenience of understanding: the constrain of the second term $\mathscr{Y}(\widetilde{\mathbf{y}}) \geq \epsilon$ is just for continuity.) Then, we further extent the domain of $\mathcal{R}_\epsilon^{\text{new}}$ to $\mathbb{C}^{N \times N}$ by Tietze extension theorem, denoted as $\mathcal{R}_\epsilon^{\text{New}}: \mathbb{C}^{N \times N} \rightarrow \mathbb{R}_\epsilon^{N \times N} \in C$. 

Finally, we verify that $\mathcal{R}_\epsilon^{\text{New}}$ is better than $\mathcal{R}^{\text{sparse}}$, that is
$$
    \mathbb{E}_{\mathbf{a}\sim\pi_T^{\text{h}}} J(\mathbf{a}, \mathcal{R}^{\text{sparse}}) < \mathbb{E}_{\mathbf{a}\sim\pi_T^{\text{h}}} J(\mathbf{a}, \mathcal{R}_\epsilon^{\text{New}}),
$$
which contradicts to the maximum property. 

Since $\pi_T^{\text{h}}(\mathbf{a}^\star) > 0$, we only need to verify that
$$
    J(\mathbf{a}^\star, \mathcal{R}^{\text{sparse}}) < J(\mathbf{a}^\star, \mathcal{R}_\epsilon^{\text{New}}).
$$
Indeed, we choose a sufficiently small $\epsilon$ satisfying
\begin{align*}
& \mathbb{E}_{\mathbf{x} \sim \mathcal{D}} \mathbbm{1}_{\{ \mathscr{Y}(\mathbf{M}^{\mathbf{a}^\star} \mathcal{F} \mathbf{x}) \geq \epsilon \}} \cdot \left[ \mathrm{S}(\mathcal{R}^{\text{dense}}(\mathbf{M}^{\mathbf{a}^\star} \mathcal{F} \mathbf{x}), x) - \mathrm{S}(\mathcal{R}^{\text{sparse}}(\mathbf{M}^{\mathbf{a}^\star} \mathcal{F} \mathbf{x}), x) \right] \\
> \ & \frac{2}{3} (J(\mathbf{a}^\star, \mathcal{R}^{\text{dense}}) - J(\mathbf{a}^\star, \mathcal{R}^{\text{sparse}})),
\end{align*}
and
\begin{align*}
& \mathbb{E}_{\mathbf{x} \sim \mathcal{D}} \mathbbm{1}_{\{ \mathscr{Y}(\mathbf{M}^{\mathbf{a}^\star} \mathcal{F} \mathbf{x}) < \epsilon \}} < \frac{1}{3} (J(\mathbf{a}^\star, \mathcal{R}^{\text{dense}}) - J(\mathbf{a}^\star, \mathcal{R}^{\text{sparse}})).
\end{align*}
Then, we have
\begin{align*}
& J(\mathbf{a}^\star, \mathcal{R}_\epsilon^{\text{New}}) - J(\mathbf{a}^\star, \mathcal{R}^{\text{sparse}}) \\
= \ & \mathbb{E}_{\mathbf{x} \sim \mathcal{D}} \left[ \mathrm{S}(\mathcal{R}_\epsilon^{\text{New}}(\mathbf{M}^{\mathbf{a}^\star} \mathcal{F} \mathbf{x}), x) - \mathrm{S}(\mathcal{R}^{\text{sparse}}( \mathbf{M}^{\mathbf{a}^\star} \mathcal{F} \mathbf{x}), x) \right] \\
= \ & \mathbb{E}_{\mathbf{x} \sim \mathcal{D}} \mathbbm{1}_{\{ \mathscr{Y}(\mathbf{M}^{\mathbf{a}^\star} \mathcal{F} \mathbf{x}) \geq \epsilon \}} \cdot \left[ \mathrm{S}(\mathcal{R}^{\text{dense}}( \mathbf{M}^{\mathbf{a}^\star} \mathcal{F} \mathbf{x}), x) - \mathrm{S}(\mathcal{R}^{\text{sparse}}( \mathbf{M}^{\mathbf{a}^\star} \mathcal{F} \mathbf{x}), x) \right] \\
\ & + \mathbb{E}_{\mathbf{x} \sim \mathcal{D}} \mathbbm{1}_{\{ \mathscr{Y}(\mathbf{M}^{\mathbf{a}^\star} \mathcal{F} \mathbf{x}) < \epsilon \}} \cdot \left[ \mathrm{S}(\mathcal{R}_\epsilon^{\text{New}}(\mathbf{M}^{\mathbf{a}^\star} \mathcal{F} \mathbf{x}), x) - \mathrm{S}(\mathcal{R}^{\text{sparse}}(\mathbf{M}^{\mathbf{a}^\star} \mathcal{F} \mathbf{x}), x) \right] \\
\geq \ & \mathbb{E}_{\mathbf{x} \sim \mathcal{D}} \mathbbm{1}_{\{ \mathscr{Y}(\mathbf{M}^{\mathbf{a}^\star} \mathcal{F} \mathbf{x}) \geq \epsilon \}} \left[ \mathrm{S}(\mathcal{R}^{\text{dense}}(\mathbf{M}^{\mathbf{a}^\star} \mathcal{F} \mathbf{x}), x) - \mathrm{S}(\mathcal{R}^{\text{sparse}}(\mathbf{M}^{\mathbf{a}^\star} \mathcal{F} \mathbf{x}), x) \right] \\
\ &  - \mathbb{E}_{\mathbf{x} \sim \mathcal{D}} \mathbbm{1}_{\{ \mathscr{Y}(\mathbf{M}^{\mathbf{a}^\star} \mathcal{F} \mathbf{x}) < \epsilon \}} \\
> \ & (\frac{2}{3} - \frac{1}{3})(J(\mathbf{a}^\star, \mathcal{R}^{\text{dense}}) - J(\mathbf{a}^\star, \mathcal{R}^{\text{sparse}})) \\ 
= \ & \frac{1}{3} (J(\mathbf{a}^\star, \mathcal{R}^{\text{dense}}) - J(\mathbf{a}^\star, \mathcal{R}^{\text{sparse}})) \\ 
> \ & 0. 
\end{align*}

Therefore, it is contradicts to the maximum property of $\mathcal{R}^{\text{sparse}}$. Therefore, for $\forall \Vert \mathbf{a} \Vert_1 = T$, 
$$
J(\mathbf{a}, \mathcal{R}^{\text{dense}}) \leq J(\mathbf{a}, \mathcal{R}^{\text{sparse}}). 
$$

The rest proof is similar to the proof of theorem \ref{theorem:non-deterministic}. Let $v_{\pi}^{\text{dense}}(\mathbf{x}, \mathbf{M}_t) = v_{\pi, \mathcal{R}^{\text{dense}}}^{\text{dense}}(\mathbf{x}, \mathbf{M}_t)$ and $v_{\pi}^{\text{sparse}}(\mathbf{x}, \mathbf{M}_t) = v_{\pi, \mathcal{R}^{\text{sparse}}}^{\text{sparse}}(\mathbf{x}, \mathbf{M}_t)$ as defined in theorem \ref{theorem:non-deterministic}. We define $v_{\pi}^{\text{dense}}(\mathcal{D}, \mathbf{M}_t) = \mathbb{E}_{\mathbf{x}\sim\mathcal{D}} v_{\pi}^{\text{dense}}(\mathbf{x}, \mathbf{M}_t)$ and $v_{\pi}^{\text{sparse}}(\mathcal{D}, \mathbf{M}_t) = \mathbb{E}_{\mathbf{x}\sim\mathcal{D}} v_{\pi}^{\text{sparse}}(\mathbf{x}, \mathbf{M}_t)$. Indeed, $v_{\pi}^{\text{dense}}(\mathcal{D}, \mathbf{M}_0)$ is the objective function of \eref{equ:dense-suboptim} and $v_{\pi}^{\text{sparse}}(\mathcal{D}, \mathbf{M}_0)$ is the objective function of \eref{equ:sparse-suboptim}. We want to prove:
$$
\sup_{\pi \in C(\mathbb{R}^{N\times N}, \Delta^N)} v_{\pi}^{\text{dense}}(\mathcal{D}, \mathbf{M}_0) \leq \sup_{\pi \in C(\mathbb{C}^{N\times N}, \Delta^N)} v_{\pi}^{\text{sparse}}(\mathcal{D}, \mathbf{M}_0). 
$$
For $\forall \epsilon > 0$, there exists a sampler $\pi_\epsilon$ satisfying
$$
\sup_{\pi \in C(\mathbb{R}^{N\times N}, \Delta^N)} v_{\pi}^{\text{dense}}(\mathcal{D}, \mathbf{M}_0) - v_{\pi_\epsilon}^{\text{dense}}(\mathcal{D}, \mathbf{M}_0) < \epsilon. 
$$
Let $\pi_\epsilon^* \in C(\mathbb{C}^{N\times N}, \Delta^N)$ defined as
$$
\pi_\epsilon^*(\cdot \mid \mathbf{y}) = \pi_\epsilon(\cdot \mid \mathcal{R}_\epsilon(\mathbf{y})), \quad \forall \mathbf{y} \in \mathbb{C}^{N\times N}. 
$$

We have the recursion formula for both $v_{\pi}^{\text{dense}}(\mathbf{x}, \mathbf{M}_0)$ and $v_{\pi}^{\text{sparse}}(\mathbf{x}, \mathbf{M}_0)$:
$$
v_{\pi}^{\text{dense}}(\mathbf{x}, M_t) = \sum_{a_t=1}^{N} \pi(a_t \mid \mathbf{x}_t) \cdot v_{\pi}(\mathbf{x}, \mathbbm{1}(\mathbf{M}_t+\mathbf{M}^{a_t})), 
$$$$
v_{\pi}^{\text{sparse}}(\mathbf{x}, M_t) = \sum_{a_t=1}^{N} \pi(a_t \mid \mathbf{y}_t) \cdot v_{\pi, \mathcal{R}}(\mathbf{x}, \mathbbm{1}(\mathbf{M}_t+\mathbf{M}^{a_t})). 
$$

Since the similarity metric function $\mathrm{S}$ is bounded by $[0, 1]$, $v_{\pi} \in [0, 1]$. Recursively, we have 
\begin{align*}
    v^{\text{dense}}_{\pi_\epsilon}(\mathcal{D}, \mathbf{M}_T) & = \mathbb{E}_{\mathbf{x}\sim\mathcal{D}} \mathrm{S}(\mathcal{R}^{\text{dense}}(\mathbf{M}_T \odot \mathcal{F}(\mathbf{x})), \mathbf{x}) \\
    & \leq \mathbb{E}_{\mathbf{x}\sim\mathcal{D}} \left[ \mathrm{S}(\mathcal{R}^{\text{sparse}}(\mathbf{M}_T \odot \mathcal{F}(\mathbf{x})), \mathbf{x}) \right] \\
    & = v^{\text{sparse}}_{\pi_\epsilon^*}(\mathcal{D}, \mathbf{M}_T). 
\end{align*}
(The inequality is because $J(\mathbf{a}, \mathcal{R}^{\text{dense}}) \leq J(\mathbf{a}, \mathcal{R}^{\text{sparse}})$ where we set $\mathbf{M}_T = \mathbf{1} \cdot \mathbf{a}^T$)
\begin{align*}
    & v^{\text{dense}}_{\pi_\epsilon}(\mathcal{D}, \mathbf{M}_{T-1}) \\ 
    = \ & \mathbb{E}_{\mathbf{x}\sim\mathcal{D}} \sum_{a_{T-1}=1}^{n}  \pi_\epsilon(a_t \mid \mathbf{x}_{T-1}) \cdot v^{\text{dense}}_{\pi_\epsilon}(\mathbf{x}, \mathbbm{1}(\mathbf{M}_{T-1}+\mathbf{M}^{a_{T-1}})) \\
    = \ & \mathbb{E}_{\mathbf{x}\sim\mathcal{D}} \sum_{a_{T-1}=1}^{n}  \pi_\epsilon^*(a_t \mid \mathbf{y}_{T-1})\cdot v^{\text{dense}}_{\pi_\epsilon}(\mathbf{x}, \mathbbm{1}(\mathbf{M}_{T-1}+\mathbf{M}^{a_{T-1}})) \\
    \leq \ & \mathbb{E}_{x\sim\mathcal{D}} \sum_{a_t=1}^{n} \pi_\epsilon^*(a_t \mid \mathbf{y}_{T-1}) \cdot [ v^{\text{sparse}}_{\pi_\epsilon^*}(\mathbf{x}, \mathbbm{1}(\mathbf{M}_{T-1}+\mathbf{M}^{a_{T-1}}))] \\
    = \ & v^{\text{sparse}}_{\pi_\epsilon^*}(\mathcal{D}, \mathbf{M}_{T-1}), 
\end{align*}
$$
\cdots
$$

Finally, we have
$$
v^{\text{dense}}_{\pi_\epsilon}(\mathcal{D}, \mathbf{M}_0) \leq v^{\text{sparse}}_{\pi_\epsilon^*}(\mathcal{D}, \mathbf{M}_0),
$$
which means that for $\forall \epsilon$, there exists a pair of $\pi_\epsilon^* \in C(\mathbb{C}^{N\times N}, \Delta^N)$ and $\mathcal{R}_\epsilon \in C(\mathbb{C}^{N\times N}, \mathbb{R}^{N\times N})$ satisfying
$$
\sup_{\pi \in C(\mathbb{R}^{N\times N}, \Delta^N)} v_{\pi, \mathcal{R}}^{\text{dense}}(\mathcal{D}, \mathbf{M}_0) < v_{\pi_\epsilon^*}(\mathcal{D}, \mathbf{M}_0) + \epsilon.
$$

Because of the arbitrariness of $\epsilon$, the statement holds! 

\qedsymbol
\end{proof}

\textbf{Remark.} We briefly show a counter-example of theorem \ref{theorem:improved-dynamic} without the first assumption that '$\pi_{T+\lVert \mathbf{M}_0 \rVert_\infty}^{\text{h}}(\mathbf{a}) > 0$ for all binary column vectors $\mathbf{a}$ satisfying $\Vert \mathbf{a} \Vert_1 = T+\lVert \mathbf{M}_0 \rVert_\infty$'. Similarly, we let $\mathbf{M}_0 = \mathbf{0}$. Suppose that there exists $\mathbf{a}^\star$ that $\Vert \mathbf{a}^\star \Vert_1 = T$ which satisfies $\pi_T(\mathbf{a}^\star)=0$. Let dataset $\mathcal{D} = \{\mathbf{x}\}$, where $\mathbf{M}^{\mathbf{a}^\star} \odot \mathcal{F}(\mathbf{x}) \neq \mathbf{0}$ and $\mathbf{M}^{\mathbf{1}-\mathbf{a}^\star} \odot \mathcal{F}(\mathbf{x}) > \mathbf{0}$. Then, let $\mathcal{R}^{\text{sparse}}$ be the solution of the optimization problem \eref{equ:R-sparse}. Moreover, let $\mathcal{R}^{\text{dense}}(\mathbf{M}^{\mathbf{a}^\star}\odot\mathcal{F}(\mathbf{x})) = \mathbf{x}$ and $\mathcal{R}^{\text{sparse}}(\mathbf{M}^{\mathbf{a}^\star}\odot\mathcal{F}(\mathbf{x})) = \mathbf{0}$. The construction in the above proof guarantees the continuity of the two reconstructors. Therefore, that $\pi^{\text{dense}} = \mathbf{a}^\star$ leads to the fact that $\mathcal{R}^{\text{dense}}$ can do reconstructs perfectly, while $\mathcal{R}^{\text{sparse}}$ can not.

\subsection{Proof of proposition \ref{prop:sparse-derivative}}
\label{appen-subsec:proof-sparse-derivative}
Our proof mainly utilizes the fact that taking the derivative and taking the expectation can be interchanged. 
\begin{proof}
Here we consider the sparse-reward POMDP. 
\begin{align*}
    \nabla_{\theta_\mathcal{R}} J_T^{\text{sparse}} (\mathbf{x}) & = \nabla_{\theta_\mathcal{R}} \mathbb{E}_{\{a_t\}_{t=0}^{T-1} \sim \pi} \left[ \mathrm{S}(\mathcal{R}(\mathbf{y}_T; \theta_\mathcal{R}), \mathbf{x} \right] \\
    & = \mathbb{E}_{\{a_t\}_{t=0}^{T-1} \sim \pi} \left[ \nabla_{\theta_{\mathcal{R}}} \mathrm{S}(\mathcal{R}(\mathbf{y}_T; \theta_\mathcal{R}), \mathbf{x} \right]. 
\end{align*}
The second equality arises from interchanging the differentiation and expectation order. This process is justified as: firstly, the differentiation concerning $\theta_\mathcal{R}$ is $\pi$-independent; secondly, $\pi$ represents a discrete distribution, making the expectation equivalent to a finite sum. These conditions enable the application of the Dominated Convergence Theorem, sanctioning our operation switch, hence confirming the second equality. 

\qedsymbol
\end{proof}

\subsection{Alternate Training for Dense-reward POMDP Fails}
\label{appen-subsec:dense-derivative}
We now demonstrate that the proposed alternate training framework is not suitable for the dense-reward POMDP. We start by calculating the derivative of $J^{\text{dense}}$ w.r.t the $\theta_\mathcal{R}$. 
\begin{proposition}
The derivative of $J^{\text{dense}}$ w.r.t. $\theta_{\mathcal{R}}$ is
\begin{equation}
\begin{split}
    \nabla_{\theta_{\mathcal{R}}} J_T^{\text{dense}}(\mathbf{x}) = & \sum_{t=0}^{T-1} \mathbb{E}_{\{a_s\}_{s=0}^{t-1} \sim \pi} \left[ \sum_{a_t} \nabla_{\theta_{\mathcal{R}}} \pi(a_t \mid \mathcal{R}(\mathbf{y}_t; \theta_\mathcal{R})) \cdot q_t \right] \\
    & + \mathbb{E}_{\{a_t\}_{t=0}^{T-1} \sim \pi} \left[ \nabla_{\theta_{\mathcal{R}}} \mathrm{S}(\mathcal{R}(\mathbf{y}_T; \theta_\mathcal{R}), \mathbf{x}) \right],
\end{split}
\label{equ:dense-derivative}
\end{equation}
where $\{a_s\}_{s=0}^{t-1} \sim \pi$ means a sequential acquisition according to \eref{equ:dense-optim-constrain}, $\mathbf{y}_t = \mathbf{y}_t (\mathbf{M}_0, \{a_s\}_{s=0}^{t-1}) = (\mathbf{M}_0 + \sum_{s=0}^{t-1} \mathbf{M}^{a_s}) \odot \mathbf{y}$, and $q_t = q(\mathbf{x}, \mathbf{M}_t, a_t)$ is the Q-function of the dense-reward POMDP for $t = 1, \cdots, T$. 
\label{prop:dense-derivative}
\end{proposition}

\begin{proof}
Here we consider the dense-reward POMDP. Let $v(\mathbf{x}, \mathbf{M}_t)$ be the value function and $q(\mathbf{x}_t, a_t)$ be the Q function at timestep $t$, defined as
$$
v(\mathbf{x}, \mathbf{M}_t) = \mathbb{E}_{\{a_s\}_{s=t}^{T-1} \sim \pi} \sum_{s=t+1}^T r_s, \quad \text{where} \ \mathbf{y}_t = \mathbf{M}_t \odot \mathcal{F}(\mathbf{x}),
$$$$
q(\mathbf{x}, \mathbf{M}_t, a_t) = r_{t+1} + v(\mathbf{x}, \mathbf{M}_{t+1}), \quad \text{where} \ \mathbf{M}_{t+1} = \mathbbm{1}(\mathbf{M}_t + \mathbf{M}^{a_t}).
$$

Indeed, 
$$
J_T^{\text{dense}}(\mathbf{x}) = v(\mathbf{x}, \mathbf{M}_0) + \mathrm{S}(\mathcal{R}(\mathbf{y}_0), \mathbf{x}).
$$

For $t=0, 1, \cdots, T-1$, we have
\begin{align*}
    & \nabla_{\theta_\mathcal{R}} v(\mathbf{x}, \mathbf{M}_t) \\
    = & \nabla_{\theta_\mathcal{R}} \left[ \sum_{a_t} \pi(a_t \mid \mathbf{x}_t) \cdot q(\mathbf{x}, \mathbf{M}_t, a_t) \right] \\
	= & \sum_{a_t} \left[ \nabla_{\theta_\mathcal{R}} \pi \cdot q_t + \pi \cdot \nabla_{\theta_\mathcal{R}} q_t \right] \\
	= & \sum_{a_t} \left[ \nabla_{\theta_\mathcal{R}} \pi \cdot q_t + \pi \cdot \nabla_{\theta_\mathcal{R}} (r_{t+1} + v(\mathbf{x}, \mathbf{M}_{t+1})) \right] \\ 
	= & \sum_{a_t} [ \nabla_{\theta_\mathcal{R}} \pi \cdot q_t + \pi \cdot \nabla_{\theta_\mathcal{R}} (\mathrm{S}(\mathcal{R}(\mathbf{y}_{t+1}), \mathbf{x}) + v(\mathbf{x}, \mathbf{M}_{t+1})) ] - \nabla_{\theta_\mathcal{R}} \mathrm{S}(\mathcal{R}(\mathbf{y}_t), \mathbf{x}),
\end{align*}
and
$$
\nabla_{\theta_\mathcal{R}} v(\mathbf{x}, \mathbf{M}_T) = 0.
$$

Finally, we have
\begin{align*}
    & \nabla_{\theta_{\mathcal{R}}} J_T^{\text{dense}}(\mathbf{x}) \\
    = \ & \sum_{t=0}^{T-1} \sum_{a_0, \cdots, a_{t-1}} \operatorname{Pr}_\pi (\mathbf{y}_0 \xrightarrow{\{a_s\}_{s=0}^{t-1}} \mathbf{y}_t) \cdot \left[ \sum_{a_t} \nabla_{\theta_{\mathcal{R}}} \pi(a_t \mid \mathcal{R}(\mathbf{y}_t)) \cdot q_t \right] \\
    & + \sum_{a_0, \cdots, a_{T-1}} \operatorname{Pr}_\pi (\mathbf{y}_0 \xrightarrow{\{a_t\}_{t=0}^{T-1}} \mathbf{y}_T) \cdot \nabla_{\theta_{\mathcal{R}}} \mathrm{S}(\mathbf{x}_T, \mathbf{x}) \\
    = \ & \sum_{t=0}^{T-1} \mathbb{E}_{\{a_s\}_{s=0}^{t-1} \sim \pi} \left[ \sum_{a_t} \nabla_{\theta_{\mathcal{R}}} \pi(a_t \mid \mathcal{R}(\mathbf{y}_t)) \cdot q_t \right] + \mathbb{E}_{\{a_t\}_{t=0}^{T-1} \sim \pi} \left[ \nabla_{\theta_{\mathcal{R}}} \mathrm{S}(\mathcal{R}(\mathbf{y}_T), \mathbf{x}) \right].
\end{align*}
\label{proof:dense-derivation}

\qedsymbol
\end{proof}

Although proposition \ref{prop:dense-derivative} gives the derivative of dense-reward objective function w.r.t parameters of the reconstructor, we empirically find that training reconstructor with the derivative by gradient-based methods fails. The reason may be that the sequential acquisition $\{a_t\}_{t=0}^{t-1} \sim \pi$ depends on not only $\pi$ but also $\mathcal{R}$. Changing the reconstructor results to an unpredictable change of the sampling strategy, leading to the failure. From another perspective, the reconstructor can also be viewed as an RL agent, so gradient-based methods hardly work.

\section{Implementation Details}
\label{appen-sec:implementation}
In this section, We will show the implementation details of all methods. In \sref{appen-subsec:basic-setting}, we show the basic settings of all algorithms. In \sref{appen-subsec:implementation-ours}, we show the implementation details of L2S and L2SR. In \sref{appen-subsec:implementation-baseline}, we show the implementation details of competing methods. 

\subsection{Basic Settings}
\label{appen-subsec:basic-setting}
Both our methods and competing methods, are performed with the same basic settings. 

\textbf{SSIM Values}. 
SSIM hyperparameters are kept to their original values in \cite{1284395}. The dynamic range is set to the maximum pixel value of the ground truth volume. 

\textbf{Loss Function}. 
Since we leverage the SSIM value as the similarity metric, for consistency, all algorithms use negative SSIM values as the loss.

\subsection{Our Methods}
\label{appen-subsec:implementation-ours}
Pre-training the reconstructor is the same as the competing method `Random', as detailed in the following subsection. We implement the A2C algorithm based on the stable-baselines3 \cite{stable-baselines3} which is a popular framework for reliable implementations of RL algorithms. The learning rate is set to 0.0003, the update timestep is set to $T$, and the other hyperparameters are to the default values. The reconstructor is trained by Adam optimizer for 10 epochs in the alternating training stage. The learning rate for both $\theta_{\pi}$ and $\theta_{\mathcal{R}}$ decays at a rate of 3 at each alternate training. We refer you to our official code for more details. 

\subsection{Competing Methods}
\label{appen-subsec:implementation-baseline}
\textbf{Random}. 
We train the reconstructor with random policy by Adam \cite{kingma2014adam} optimizer for 50 epochs. Other hyperparameters inherit the fastMRI repository. We always use early stopping to get the best model on the validation dataset. 

\textbf{PG-MRI}. 
We use the best performing $\gamma=0.9$ Non-Greedy method in \cite{NEURIPS2020_daed2103}. We train the sampler with its default hyperparameters. The fixed reconstructor is pre-trained by "Random" heuristic sampling policy. Notice that the reconstructor pre-trained is different from \cite{NEURIPS2020_daed2103}, since the loss is negative SSIM instead of $L_1$-norm and the training dataset is half of the volumes instead of full. In addition, we use early stopping based on the validation set rather than training for the full 50 epochs for better reconstruction performances. 

\textbf{Greedy Oracle}. 
The one-step greedy oracle policy has access to the ground truth for reference. We choose the action that increases similarity the most at test time with a fixed reconstructor the same as above. 

\textbf{LOUPE}. 
This method jointly trains the parametric probability mask $\mathbf{p}$ and the reconstructor end-to-end via BP without pre-training. We use the implementation by \cite{yin2021end} and train with its default hyperparameters. 

\textbf{$\tau$-Step Seq}. 
This method jointly trains the sampler and the reconstructor end-to-end via BP without pre-training. We use the default hyperparameters in \cite{yin2021end}.

\section{Supplemental Experimental Results}

\subsection{RL Algorithms}
We have experimented with PPO \cite{schulman2017proximal}, an advanced RL algorithm, to learn a better sampling policy. However, we have found that PPO has slow convergence in practice and only slightly improves the SSIM values (around $0.001 \sim 0.002$). Therefore, we have decided not to use it in our proposed method.

\subsection{Training Cost}
Our proposed L2SR framework necessitates a longer training period for two primary reasons. Firstly, the employment of a sparse reward POMDP demands an extended exploration phase due to its complexity compared to the dense reward POMDP (of course it is more efficient in calculating rewards). Secondly,  the alternating optimization strategy inherent in our approach, contrasting with simpler, one-step methods, further extends the training time. Taken together, our training time may be 3 to 5 times longer than existing dynamic methods. However, in practice we are more concerned with the inference time than the training cost.

\section*{References}
\bibliographystyle{unsrt}
\bibliography{refbib}

\end{document}